\documentclass[11pt]{article}
\usepackage{latexsym,amssymb,amsmath} 
\usepackage{path}
\usepackage{url}
\usepackage{verbatim}

\usepackage{natbib}
\usepackage{amsfonts}
\usepackage{amsfonts}
\usepackage{amsfonts}
\usepackage{array}
\usepackage{mathrsfs}
\usepackage{amsfonts}
\usepackage{amsmath}
\usepackage{amssymb}
\usepackage{amsthm}
\usepackage{bm}
\usepackage{appendix}

\usepackage{graphicx} 
\usepackage{subfigure}

\graphicspath{{figures/}}

\usepackage[lined,boxed,ruled, commentsnumbered]{algorithm2e}

\newtheorem{lemma}{Lemma}
\newtheorem{theorem}{Theorem}

\newtheorem{definition}{Definition}
\newtheorem{remark}{Remark}
\newtheorem{assumption}{Assumption}

\newcommand{\supp}{\text{supp}}

\newcommand{\argmax}{\mathop{\arg\max}}

\newcommand{\Tr}{\text{Tr}}
\newcommand{\Trunc}{\text{Truncate}}
\newcommand{\otherwise}{\text{otherwise}}
\newcommand{\tabincell}[2]{\begin{tabular}{@{}#1@{}}#2\end{tabular}}


\numberwithin{equation}{section}
\numberwithin{table}{section}
\numberwithin{figure}{section}

\newcommand{\eigenmax}{\lambda_{\max}}
\newcommand{\spectral}[1]{\rho(#1)}

\newcommand{\junk}[1]{{}}

\newlength{\fwtwo} 

\DeclareGraphicsExtensions{.jpg}

\title{Truncated Power Method for Sparse Eigenvalue Problems}
\author{
  Xiao-Tong Yuan \\
  Department of Statistics
  Rutgers University \\
  New Jersey, 08816 \\
  \texttt{xyuan@stat.rutgers.edu} \\
  \and
  Tong Zhang \\
  Department of Statistics
  Rutgers University \\
  New Jersey, 08816 \\
  \texttt{tzhang@stat.rutgers.edu}
}
\date{}

\begin{document}

\maketitle

\begin{abstract}
This paper considers the sparse eigenvalue problem, which is to
extract dominant (largest) sparse eigenvectors with at most $k$ non-zero
components. We propose a simple yet effective solution called {\em truncated power
method} that can approximately solve the underlying nonconvex optimization problem.
A strong sparse recovery result is proved for the truncated power method, and this theory
is our key motivation for developing the new algorithm.
The proposed method is tested on applications such as
sparse principal component analysis and the densest $k$-subgraph problem.
Extensive experiments on several synthetic and real-world large scale datasets demonstrate the
competitive empirical performance of our method.
\end{abstract}



\section{Introduction}
\label{sect:Sparse Eigenvalue Problem}

Given a $p \times p$ symmetric positive semidefinite matrix $A$, the
\emph{largest $k$-sparse eigenvalue} problem aims to maximize the
quadratic form $x^\top Ax$ with a sparse unit vector $x \in
\mathbb{R}^p$ with no more than $k$ non-zero elements:
\begin{equation}\label{prob:spase_eigenvalue}
\eigenmax (A,k) = \max_{x \in \mathbb{R}^p}  x^\top Ax, \ \ \ \ \ \
\ \text{subject to } \|x\|=1, \|x\|_0 \le k,
\end{equation}
where $\|\cdot\|$ denotes the $\ell_2$-norm, and
$\|\cdot\|_0$ denotes the $\ell_0$-norm which counts the
number of non-zero entries in a vector. The sparsity is controlled
by the values of $k$ and can be viewed as a design parameter. In
machine learning applications, e.g., principal component analysis,
this problem is motivated from the following perturbation
formulation of matrix $A$:
\begin{equation}
A = \bar{A} + E,
\label{eq:error-decomp}
\end{equation}
where $A$ is the empirical covariance matrix,
$\bar{A}$ is the true covariance matrix, and $E$ is a random perturbation
due to having only a finite number of empirical samples.
If we assume that the largest eigenvector
$\bar{x}$ of $\bar{A}$ is sparse, then a natural question is to
recover $\bar{x}$ from the noisy observation $A$ when the error $E$
is ``small''. In this context, the problem~\eqref{prob:spase_eigenvalue}
is also referred to as sparse principal component analysis (sparse PCA).

In general, problem~\eqref{prob:spase_eigenvalue} is non-convex. In fact, it is also NP-hard because it can be
reduced to the subset selection problem for ordinary least squares regression~\citep{Moghaddam-2006},
which is known to be NP hard.
Various researchers have proposed approximate optimization methods: some are based on greedy
procedures~\citep[e.g.,][]{Moghaddam-2006,Jolliffe-2003,Aspremont-PathSPCA-2008},
and some others are based on various types of convex relaxation or
reformulation~\citep[e.g.,][]{Aspremont-2007,Zou-2006,Journee-2010}.
Although many algorithms have been proposed, almost no satisfactory theoretical results exist for this problem.
The only exception is the analysis of the convex relaxation method \citep{Aspremont-2007}
by \citet{Amash-2009} under the high dimensional spiked covariance model~\citep{Johnstone-2001}.
However, the result was concerned with variable selection consistency under a very simple and specific example
with limited general applicability.

This paper proposes a new computational procedure called {\em
truncated power iteration method} that approximately solves
\eqref{prob:spase_eigenvalue}. This method is similar to the
classical power method, with an additional truncation operation to
ensure sparsity. We show that if the true matrix $\bar{A}$ has a
sparse (or approximately sparse) dominant eigenvector $\bar{x}$,
then under appropriate assumptions, this algorithm can recover
$\bar{x}$ when the spectral norm of sparse submatrices of the
perturbation $E$ is small. Moreover, this result can be proved under
relative generality without restricting ourselves to the rather
specific spiked covariance model. Therefore our analysis provides
strong theoretical support for this new method, and this
differentiates our proposal from previous studies. We have applied
the proposed method to sparse PCA and to the densest $k$-subgraph
finding problem (with proper modification). Extensive experiments on
synthetic and real-world large scale datasets demonstrate both the
competitive sparse recovering performance and the computational
efficiency of our method.

It is worth mentioning that the  truncated power method developed in
this paper can also be applied to the \emph{smallest $k$-sparse
eigenvalue} problem given by:
\begin{equation}\label{prob:spase_eigenvalue_smallest}
\lambda_{\min} (A,k) = \min_{x \in \mathbb{R}^p} \; x^\top Ax, \ \ \
\ \ \ \ \text{subject to } \|x\|=1, \quad \|x\|_0 \le k , \nonumber
\end{equation}
which also has many applications in machine learning.

\subsection{Notation}
Let $\mathbb{S}^p=\{A \in \mathbb{R}^{p \times p} \mid A = A^\top\}$
denote the set of symmetric matrices, and $\mathbb{S}_{+}^p=\{A \in
\mathbb{S}^{p} , A \succeq 0\}$ denote the cone of symmetric,
positive semidefinite (PSD) matrices. For any $A \in \mathbb{S}^p$,
we denote its eigenvalues by $\lambda_{\min}(A) = \lambda_p (A) \le
\cdots \le \lambda_1(A) = \eigenmax(A)$. We use $\spectral{A}$ to
denote the spectral norm of $A$, which is
$\max\{|\lambda_{\min}(A)|,|\eigenmax(A)|\}$, and define
$\spectral{A,s} := \max\{|\lambda_{\min}(A,s)|,|\eigenmax(A,s)|\}$.
The $i$-th entry of vector $x$ is denoted by $[x]_i$ while
$[A]_{ij}$ denotes the element on the $i$-th row and $j$-th column
of matrix $A$. We denote by $A_k$ any $k \times k$ principal
submatrix of $A$ and by $A_F$ the principal submatrix of $A$ with
rows and columns indexed in set $F$. If necessary, we also denote
$A_F$ as the restriction of $A$ on the rows and columns indexed in
$F$. Let $\|x\|_{p}$ be the $\ell_p$-norm of a vector $x$. In
particular, $\|x\|_2=\sqrt{x^\top x}$ denotes the Euclidean norm,
$\|x\|_1 = \sum_{i=1}^d |[x]_i|$ denotes the $\ell_1$-norm, and
$\|x\|_0=\#\{j: [x]_j \neq 0\}$ denotes the $\ell_0$-norm. For
simplicity, we also denote the $\ell_2$ norm $\|x\|_2$ by $\|x\|$. In
the rest of the paper, we define $Q(x) := x^\top  A x$ and let $x_*$
be the optimal solution of problem~\eqref{prob:spase_eigenvalue}. We
let $\supp(x):=\{j: \mid [x]_j \neq 0\}$ denote the support set of
the vector $x$. Given an index set $F$, we define
\[
x(F):= \argmax_{x \in \mathbb{R}^p} x^\top Ax, \ \ \ \text{subject
to } \|x\|=1, \quad \supp(x) \subseteq F.
\]
Finally, we denote by $I_{p \times p}$ the $p \times p$ identity
matrix.

\subsection{Paper Organization}

The remaining of this paper is organized as follows:
Section~\ref{sect:TPower} describes the truncated power iteration
algorithm that approximately solves
problem~\eqref{prob:spase_eigenvalue}. In
Section~\ref{sect:algorithm_analysis} we analyze the solution
quality of the proposed algorithm. Section~\ref{sect:applications}
evaluates the practical performance of the proposed algorithm in
applications of sparse PCA and the densest $k$-subgraph finding
problems. We conclude this work and discuss potential extensions in
Section~\ref{sect:conclusion}.

\section{Truncated Power Method}
\label{sect:TPower}

Since $\eigenmax(A, k )$ equals
$\eigenmax(A^*_k)$ where $A^*_k$ is the $k \times k$ principal
submatrix of $A$ with the largest eigenvalue, one may solve \eqref{prob:spase_eigenvalue}
by exhaustively enumerate all subsets of $\{1,\ldots,p\}$ of size $k$ in order to find $A^*_k$.
However, this procedure is impractical even for moderate sized $k$
since the number of subsets is exponential in $k$.

Therefore in order to solve the spare eigenvalue
problem~\eqref{prob:spase_eigenvalue} more efficiently, we consider
an iterative procedure based on the standard power method for
eigenvalue problems, while maintaining the desired sparsity for the
intermediate solutions. The procedure, presented in
Algorithm~\ref{alg:sparse_power_eig}, generates a sequence of
intermediate $k$-sparse eigenvectors $x_0, x_1, \ldots$ from an
initial sparse approximation $x_0$. At each step $t$, the
intermediate vector $x_{t-1}$ is multiplied by $A$, and then the
entries are truncated to zeros except for the largest $k$ entries.
The resulting vector is then normalized to unit length, which
becomes $x_t$. It will be assumed throughout the paper that the
cardinality $k$ of $\supp(x_*)$ is available a prior; in practice
this quantity may be regarded as a tuning parameter of the
algorithm.

\begin{definition}
Given a vector $x$ and an index set $F$, we define the truncation operation $\Trunc(x,F)$
to be the vector obtained by restricting $x$ to $F$, that is
\[
[\Trunc (x,F)]_i =  \left \{
\begin{array}{l}
[x]_i \ \ \ i \in F \\
0 \ \ \ \ \ \ \otherwise
\end{array}\right..
\]

\end{definition}

\begin{algorithm}[h]
\SetKwInOut{Input}{Input}
\SetKwInOut{Output}{Output}
\SetKwInOut{Parameters}{Parameters}
\SetKwInOut{Initialization}{Initialization}

\Input{matrix $A \in \mathbb{S}^p$, initial vector $x_0 \in \mathbb{R}^p$}
\Output{$x_t$}
\Parameters{cardinality $k \in \{1,...,p\}$}
\BlankLine Let $t=1$.

\Repeat{Convergence} {

Compute $x'_{t} = Ax_{t-1}/\|A x_{t-1}\|$.

Let $F_t = \supp(x'_t,k)$ be the indices of $x'_t$ with the largest $k$
absolute values.

Compute $\hat{x}_t = \Trunc( x'_t, F_t )$.

Normalize $x_t = \hat{x}_t/\|\hat{x}_t\|$.

$t \leftarrow t + 1$.

}

\caption{Truncated Power (TPower) Method}  \label{alg:sparse_power_eig}
\end{algorithm}

\begin{remark}
Similar to the behavior of traditional power method, if $A \in
\mathbb{S}_+^p$, then TPower tries to find the (sparse) eigenvector
of $A$ corresponding to the largest eigenvalue. Otherwise, it may
find the (sparse) eigenvector with the smallest eigenvalue if
$-\lambda_p(A)
> \lambda_1(A)$. However, this situation is easily detectable
because it can only happen when $\lambda_p(A)<0$. In such case, we
may restart TPower with $A$ replaced by an appropriately shifted
version $A + \tilde{\lambda} I_{p \times p}$.
\end{remark}


\section{Sparse Recovery Analysis}
\label{sect:algorithm_analysis}

We consider the general noisy matrix model (\ref{eq:error-decomp}),
and are specially interested in the high dimensional situation where
the dimension $p$ of $A$ is large. We assume that the noise matrix
$E$ is a dense $p\times p$ matrix such that its sparse submatrices
have small spectral norm $\spectral{E,s}$ for $s$ in the same order
of $k$. We refer to this quantity as {\em restricted perturbation
error}. However, the spectral norm of the full matrix perturbation
error $\spectral{E}$ can be large. For example, if the original
covariance is corrupted by an additive standard Gaussian iid noise
vector, then $\spectral{E,s} = O(\sqrt{s\log p / n})$, which grows
linearly in $\sqrt{s}$, instead of $\spectral{E}=O(\sqrt{p/n})$, 
which grows linearly in $\sqrt{p}$. The main advantage of the
sparse eigenvalue formulation~\eqref{prob:spase_eigenvalue} over the
standard eigenvalue formulation is that the estimation error of its
optimal solution depends on $\spectral{E,s}$ with respectively  a
small $s = O(k)$ rather than $\spectral{E}$. This linear dependency
on sparsity $k$ instead of the original dimension $p$ is analogous
to similar results for sparse regression (or compressive sensing)
such as \citep{CandTao05-rip}. In fact the restricted perturbation
error considered here is analogous to the idea of restricted
isometry property (RIP) considered in \citep{CandTao05-rip}.

The purpose of the section is to show that if matrix $\bar{A}$ has a unique sparse (or approximately sparse)
dominant eigenvector, then under suitable conditions,
TPower can (approximately) recover this eigenvector from the noisy observation $A$.

\begin{assumption}
  \label{assump:sparsity}
Assume that the largest eigenvalue of $\bar{A} \in \mathbb{S}^p$ is
$\lambda=\eigenmax(\bar{A})>0$ that is non-degenerate,  with a gap
$\Delta \lambda= \lambda - \max_{j>1}|\lambda_j(\bar{A})|$ between
the largest and the remaining eigenvalues. Moreover, assume that the
eigenvector $\bar{x}$ corresponding to the dominant eigenvalue
$\lambda$ is sparse with cardinality $\bar{k}=\|\bar{x}\|_0$.
\end{assumption}

We want to show that under Assumption~\ref{assump:sparsity}, if the spectral norm $\spectral{E,s}$
of the error matrix
is small for an appropriately chosen $s > \bar{k}$, then it is possible
to approximately recover $\bar{x}$.
Note that in the extreme case of $s=p$, this result follows directly from the standard eigenperturbation
analysis (which does not require Assumption~\ref{assump:sparsity}).

We now state our main result as below, which shows that under
appropriate conditions, the TPower method can recover the sparse
eigenvector. The final error bound is a direct generalization of standard matrix perturbation result that depends on the
full matrix perturbation error $\spectral{E}$. Here this quantity is replaced by the restricted
perturbation error $\spectral{E,s}$.

\begin{theorem}
We assume that Assumption~\ref{assump:sparsity} holds. Let
$s=2k+\bar{k}$ with $k \geq 4 \bar{k}$. Assume that $\rho(E, s) \le
\Delta \lambda/2$. Define
\[
\gamma(s):=\frac{\lambda  - \Delta \lambda +
\spectral{E,s}}{\lambda - \spectral{E,s}} < 1
, \qquad
\delta(s):=
\frac{\sqrt{2}\spectral{E,s}}{\sqrt{\spectral{E,s}^2 +
(\Delta\lambda - 2\spectral{E,s})^2}}.
\]
If $|x_0^\top \bar{x}| \geq u+\delta(s)$ for some $\|x_0\|_0 \leq k$, $\|x_0\|=1$, and $u \in [0,1]$ such that
\begin{align*}
\mu_1 =&  (1-\gamma(s)^2) u(1-u^2)/2  - \left(2\delta(s)+(\bar{k}/k)^{1/2}\right)\in (0,1) , \\
\mu_2 =&  \sqrt{(1+3(\bar{k}/k)^{1/2})(1-0.45(1-\gamma(s)^2))} < 1 ,
\end{align*}
then let $\mu=\max(\sqrt{1-\mu_1},\mu_2)$, we have
\[
\sqrt{1-|x_t^\top \bar{x}|} \leq \mu^t \sqrt{1-|x_0^\top \bar{x}|} +
\sqrt{5} \delta(s) /(1-\mu_2) .
\]
\label{thm:recovery}
\end{theorem}

\begin{remark}
  We only state our result with a relatively simple but easy to understand quantity $\rho(E,s)$, which
  we refer to as restricted perturbation error. It is analogous to the RIP concept in \citep{CandTao05-rip}, and
  is also directly comparable to the traditional full matrix perturbation error $\rho(E)$.
  While it is possible to obtain sharper results with additional quantities, we intentionally keep the theorem simple so that
  its consequence is relatively easy to interpret.
\end{remark}

\begin{remark}
  Although we state the result by assuming that the dominant eigenvector $\bar{x}$ is sparse, the theorem can also
  be applied to certain situations that $\bar{x}$ is only approximately sparse. In such case, we simply let
  $\bar{x}'$ be a $\bar{k}$ sparse approximation of $\bar{x}$. If $\bar{x}'-\bar{x}$ is sufficiently small, then
  $\bar{x}'$ is the dominant eigenvector of a symmetric matrix $\bar{A}'$ that is
  close to $\bar{A}$; hence the theorem can be
  applied with the decomposition $A = \bar{A}' + E'$ where $E'= E + A-\bar{A}'$.
\end{remark}

Note that we did not make any attempt to optimize the constants in Theorem~\ref{thm:recovery}, which are relatively large. Therefore in the discussion, we shall ignore the constants, and focus on the main message of
 Theorem~\ref{thm:recovery}. If $\rho(E,s)$ is smaller than the eigen-gap $\Delta \lambda / 2 > 0$, then
$\gamma(s)<1$ and $\delta(s)=O(\rho(E,s))$.
It follows that under appropriate conditions, as long as we can find an initial $x_0$ such that
\[
|x_0^\top \bar{x}|\geq c (\rho(E,s) + (\bar{k}/k)^{1/2})
\]
for some constant $c$, then $1-|x_t^\top \bar{x}|$ converges geometrically until
\[
\|x_t- \bar{x}\|^2 = O(\rho(E,s)^2 ) .
\]
This result is similar to the standard eigenvector perturbation
result stated in Lemma~\ref{lem:perturb} of Appendix~\ref{append:proof},
except that we replace the spectral error $\rho(E,p)$ of the full
matrix by $\rho(E,s)$ that can be significantly smaller when $s \ll
p$. To our knowledge, this is the first sparse recovery
result for the sparse eigenvalue problem in a relatively general setting.
This theorem can be considered as a strong theoretical justification of the proposed
TPower algorithm that distinguishes it from earlier algorithms
without theoretical guarantees. Specifically, the replacement of the
full matrix perturbation error $\rho(E)^2$ with $\rho(E,s)^2$ gives the
theoretical insights on why TPower works well in practice.

To illustrate our result,
we briefly describe a consequence of the theorem under the spiked covariance model of \citep{Johnstone-2001} which
was investigated by \citet{Amash-2009}.
We assume that the observations are $p$ dimensional vectors
\[
x_i = \bar{x} + \epsilon ,
\]
for $i=1,\ldots,n$, where $\epsilon \sim N(0, I_{p \times p})$.
For simplicity, we assume that $\|\bar{x}\|=1$. The true covariance is
\[
\bar{A} = \bar{x} \bar{x}^\top + I_{p \times p} ,
\]
and $A$ is the empirical covariance
\[
A = \frac{1}{n} \sum_{i=1}^n x_i x_i^\top .
\]
Let $E=A-\bar{A}$, then random matrix theory implies that with large probability,
\[
\rho(E,s) = O( \sqrt{s\ln p/n}) .
\]
Now assume that $\max_j |\bar{x}_j|$ is sufficiently large.
In this case, we can run TPower with a starting point $x_0= e_j$ for some vector
$e_j$ (where $e_j$ is the vector of zeros except the $j$-th entry being one) so that
$|e_j^\top \bar{x}|=|\bar{x}_j|$ is sufficiently large, and the assumption for the
initial vector $|x_0^\top \bar{x}|\geq c (\rho(E,s) + (\bar{k}/k)^{1/2})$ is satisfied with
$s=O(\bar{k})$.
We may run TPower with an appropriate initial vector to obtain an approximate solution
$x_t$ of error
\[
\|x_t- \bar{x}\|^2 = O(\bar{k} \ln p /n) .
\]
This is optimal. Note that our results are not directly comparable
to those of \citet{Amash-2009}, which studied support recovery.
Nevertheless, it is worth noting that if $\max_j |\bar{x}_j|$ is sufficiently large,
then our result becomes meaningful when $n=O(\bar{k} \ln p)$;
however their result requires $n=O(\bar{k}^2\ln p)$ to be meaningful, although this is for the pessimistic case of
$\bar{x}$ having equal nonzero values of $1/\sqrt{\bar{k}}$.

Finally we note that if we cannot find a large initial value with
$|x_0^\top \bar{x}|$, then it may be necessary to take a relatively
large $k$ so that the requirement $|x_0^\top \bar{x}|\geq c
((\bar{k}/k)^{1/2})$ is satisfied. With such a $k$, $\rho(E,s)$ may
be relatively large and hence the theorem indicates that $x_t$ may
not converge to $\bar{x}$ accurately. Nevertheless, as long as
$|x_t^\top \bar{x}|$ converges to a value that is not too small
(e.g., can be much larger than $|x_0^\top \bar{x}|$), we may reduce
$k$ and rerun the algorithm with $x_t$ as initial vector together
with a small $k$. In this two stage process, the vector found from
the first stage (with large $k$) is used to as the initial value of
the second stage (with small $k$). Therefore we may also regard it
as an initialization method to TPower. In practice, one may use
other methods to obtain an approximate $x_0$ to initialize TPower,
not necessarily restricted to running TPower with larger $k$. Some
practical alternatives are discussed in
Section~\ref{sect:applications}.

\section{Applications}
\label{sect:applications}

In this section, we illustrate the effectiveness of TPower method
when applied to sparse principal component analysis (sparse PCA) (in
Section~\ref{ssect:SPCA}) and the densest $k$-subgraph (DkS) finding
problem (in Section~\ref{ssect:dense_k_subgraph}). The Matlab code
for reproducing the experimental results reported in this section is
online available
at~\url{https://sites.google.com/site/xtyuan1980/publications}.

\subsection{Sparse PCA}
\label{ssect:SPCA}

Principal component analysis (PCA) is a well established tool for
dimensionality reduction and has a wide range of applications in
science and engineering where high dimensional datasets are
encountered. Sparse principal component analysis (sparse PCA) is an
extension of PCA that aims at finding sparse vectors (loading
vectors) capturing the maximum amount of variance in the data. In
recent years, various researchers have proposed various approaches
to directly address the conflicting goals of explaining variance and
achieving sparsity in sparse PCA. For instance, \citet{Zou-2006}
formulated sparse PCA as a regression-type optimization problem and
imposed the Lasso~\citep{Tibshirani-1996} penalty on the regression
coefficients. The DSPCA algorithm in~\citep{Aspremont-2007} is an
$\ell_1$-norm based semidefinite relaxation for sparse PCA.
\citet{Shen-2008} resorted to the singular value decomposition (SVD)
to compute low-rank matrix approximations of the data matrix under
various sparsity-inducing penalties. Greedy search and
branch-and-bound methods were investigated in~\citep{Moghaddam-2006}
to solve small instances of sparse PCA exactly and to obtain
approximate solutions for larger scale problems. More recently,
\citet{Aspremont-PathSPCA-2008} proposed the use of greedy forward
selection with a certificate of optimality, and \citet{Journee-2010}
studied a generalized power method to solve sparse PCA with a
certain dual reformulation of the problem. In comparison, our method
works directly in the primal domain, and the resulting algorithm is
quite different from that of \citet{Journee-2010}.

Given a sample covariance matrix, $\Sigma \in \mathbb{S}_+^p$ (or
equivalently a centered data matrix $D \in \mathbb{R}^{n \times p}$
with $n$ rows of $p$-dimensional observations vectors such that
$\Sigma = D^\top D$) and the target cardinality $k$, following
~\citep{Moghaddam-2006,Aspremont-2007,Aspremont-PathSPCA-2008}, we
formulate sparse PCA as:
\begin{equation}\label{prob:sparse_pca}
\hat{x} = \argmax_{x \in \mathbb{R}^p} x^\top  \Sigma  x, \ \ \ \ \ \ \
\text{subject to } \|x\|=1, \|x\|_0 \le k.
\end{equation}
The TPower method proposed in this paper can be directly applied to
solve the above problem. One advantage of TPower for Sparse PCA is
that it directly addresses the constraint on cardinality $k$. To
find the top $m$ rather than the top one sparse loading vectors, a
common approach in the
literature~\citep{Aspremont-2007,Moghaddam-2006,Mackey-NIPS-2008} is
to use the \emph{iterative  deflation} method for PCA: subsequent
sparse loading vectors can be obtained by recursively removing the
contribution of the previously found loading vectors from the
covariance matrix. Here we employ a projection deflation scheme
from~\citep{Mackey-NIPS-2008}, which deflates an vector $\hat{x}$
using the formula:
\[
\Sigma' = (I_{p\times p} - \hat{x}\hat{x}^\top)\Sigma(I_{p\times p}
- \hat{x}\hat{x}^\top).
\]
Obviously, $\Sigma'$ remains positive semidefinite. Moreover,
$\Sigma'$ is rendered left and right orthogonal to $\hat{x}$.

\subsubsection{Connection with Existing Sparse PCA Methods}
\label{ssect:connections}

In the setup of sparse PCA, TPower is closely related to
GPower~\citep{Journee-2010} and sPCA-rSVD~\citep{Shen-2008} which
are both power-truncation-type iterative algorithms. Indeed, GPower
and sPCA-rSVD are identical except for the initialization and
post-processing phases~\citep{Journee-2010}. Given a data matrix $D
\in \mathbb{R}^{n \times p}$, the $\ell_1$-norm version of GPower
(and equivalently sPCA-rSVD) solves the following regularized rank-1
approximation problem:
\[
\min_{x \in \mathbb{R}^p, z \in \in \mathbb{R}^n} \|D - zx^\top \|_F^2 +
\gamma \|x\|_1, \ \ \ \text{subject to } \|z\|=1,
\]
while the $\ell_0$-norm version of GPower (and sPCA-rSVD) solves the
following optimization problem:
\[
\min_{x \in \mathbb{R}^p, z \in \in \mathbb{R}^n} \|D - zx^\top \|_F^2 +
\gamma \|x\|_0, \ \ \ \text{subject to } \|z\|=1.
\]
Given the covariance matrix $\Sigma = D^\top D$, it is easy to
verify that TPower optimizes the following constrained low-1 and
semidefinite approximation problem
\[
\min_{x \in \mathbb{R}^p} \|\Sigma -  xx^\top \|_F^2, \ \ \
\text{subject to } \|x\|=1, \ \|x\|_0 \le k.
\]
Essentially, TPower, GPower and sPCA-rSVD all use certain
power-truncation type procedure to generate sparse loadings.
However, the difference between TPower and GPower (sPCA-rSVD) is
also clear: the former performs rank-1, semidefinite and sparse
approximation to covariance matrix while the latter performs rank-1
and sparse approximation to the data matrix. One important benefit
of TPower is that we are able to analyze solution quality such as
sparse recovery capability, while analogous results are not
available for GPower and sPCA-rSVD.

Our method is also related to
PathSPCA~\citep{Aspremont-PathSPCA-2008} that directly addresses the
formulation~\eqref{prob:spase_eigenvalue}. The PathSPCA method is a
greedy forward selection procedure which starts from the empty set
and at each iteration it selects the most relevant variable and adds
it to the current variable set; it then re-estimates the leading
eigenvector on the augmented variable set. Both TPower and PathSPCA
output sparse solutions with exact cardinality $k$.

\subsubsection{On Initialization}

Theorem~\ref{thm:recovery} suggests that the TPower algorithm can benefit from a good initial vector $x_0$.
In a practical implementation, the following three initialization schemes can
be considered.
\begin{itemize}
  \item[1:] One simple method is to set $[x_{0}]_j= 1$ on index $j = \argmax_{i} [A]_{ii}$ and $0$ otherwise.
    This initialization provides a $1/k$-approximation to the optimal
    value, i.e., $Q(x_0) \ge \eigenmax(A,k) / k$. Indeed, if we let
    $A^*_k$ be the $k \times k$ principle submatrix of $A$ supported on
    $\supp(x_*)$, then it is easy to verify that $[A]_{jj} \ge \Tr(A^*_k) / k
    \ge \eigenmax(A,k) / k$. In the setup of sparse PCA, this
    corresponds to initializing by selecting the variable with the largest
    variance, which is known to perform well for
    PathSPCA~\citep{Aspremont-PathSPCA-2008}. Alternatively, we may
    initialize $x_0$ as the indicator vector of the top $k$ values
    of the variances $\{[A]_{ii}\}$, as is considered by~\citet{Amash-2009}.
 \item[2:] A two-stage warm-start strategy suggested at the end of Section~\ref{sect:algorithm_analysis}.
   In the first stage we may run TPower with a relatively large $k$ and use the output
   as the initial value of the next stage with a decreased $k$. Repeat
   this procedure if necessary until the desired cardinality is reached.
   More generally, one may use other algorithms to warm start TPower.
  \item[3:] When $k \approx p$, an initialization scheme
    suggested in~\citep{Moghaddam-2006} can be employed.
  This scheme is motivated from the following observation: among all the $m$ possible $(m-1)\times(m-1)$
  principal submatrices of $A_m$, obtained by deleting the $j$-th row
  and column, there is at least one submatrix $A_{m-1} =
  A_{m\backslash j}$ whose maximal eigenvalue is a major fraction of
  its parent~\citep[see, e.g.,][]{Horn-1991}:
  \begin{equation}\label{equat:A_j_bound}
    \exists j \in \{1,...,m\}, \ \ \ \eigenmax (A_{m\backslash j})
    \ge \frac{m-1}{m} \eigenmax(A_m).
  \end{equation}
  A greedy backward elimination method is suggested using the above
  bound~\citep{Moghaddam-2006}: start with the full index set
  $\{1,...,p\}$, and sequentially delete the variable $j$ which yields
the maximum $\eigenmax(A_{m\backslash j})$ until only $k$ elements
remain. It is immediate from the bound~\eqref{equat:A_j_bound} that
this procedure will guarantee a $k/p$-approximation to the optimal
objective value. This scheme works well for relatively small $p$.
When $p$ is large, however, such a greedy initialization scheme will
be computationally prohibitive since it involves $(p-k)p$ times of
dominant eigenvalue calculation for matrices of scale $O(p\times
p)$.
\end{itemize}

\subsubsection{Results on Toy Dataset}

To illustrate the sparse recovering performance of TPower, we apply
the algorithm to a synthetic dataset drawn from a sparse PCA model.
We follow the same procedure proposed by~\citep{Shen-2008} to
generate random data with a covariance matrix having sparse
eigenvectors. To this end, a covariance matrix is first synthesized
through the eigenvalue decomposition $\Sigma = VDV^\top $, where the
first $m$ columns of $V \in \mathbb{R}^{p \times p}$ are
pre-specified sparse orthonormal vectors. A data matrix $X \in
\mathbb{R}^{n \times p}$ is then generated by drawing $n$ samples
from a zero-mean normal distribution with covariance matrix
$\Sigma$, that is $X \sim \mathcal {N}(0, \Sigma)$. The empirical
covariance $\hat{\Sigma}$ matrix is then estimated from data $X$ as
the input for TPower.

Consider a setup with $p=500$, $n=50$, and the first $m=2$ dominant
eigenvectors of $\Sigma$ are sparse. Here the first two dominant
eigenvectors are specified as follows:
\[
[v_1]_i = \left \{
\begin{array}{l}
\frac{1}{\sqrt{10}}, \ \ \ i=1,...,10 \\
0,\ \ \ \ \ \ \ \text{otherwise}
\end{array}\right., \ \ \
[v_2]_{i} = \left \{
\begin{array}{l}
\frac{1}{\sqrt{10}}, \ \ \ i=11,...,20 \\
0,\ \ \ \ \ \ \ \text{otherwise}
\end{array}\right. .
\]
The remaining eigenvectors $v_j$ for $j \ge 3$ are chosen
arbitrarily, and the eigenvalues are fixed at the following values:
\[
\left \{
\begin{array}{l}
\lambda_{1} = 400, \\
\lambda_{2} = 300, \\
\lambda_{j} = 1, \ \ \ j= 3,...,500.
\end{array}\right.
\]
We generate $500$ data matrices and employ the TPower to compute two
unit-norm sparse loading vectors $u_1, u_2 \in \mathbb{R}^{500}$,
which are hopefully close to $v_1$ and $v_2$. Our method is
compared on this dataset with a greedy algorithm
PathPCA~\citep{Aspremont-PathSPCA-2008}, a power-iteration-type
method GPower~\citep{Journee-2010}, a sparse regression based method
SPCA~\citep{Zou-2006}, and the standard PCA. For GPower, we test its
two block versions $\text{GPower}_{\ell_1,m}$ and
$\text{GPower}_{\ell_0,m}$ with $\ell_1$-norm and $\ell_0$-norm
penalties, respectively. Here we do not involve two representative
sparse PCA algorithms, the sPCA-rSVD~\citep{Shen-2008} and the
DSPCA~\citep{Aspremont-2007}, in our comparison since the former is
shown to be identical to GPower up to initialization and
post-processing phases~\citep{Journee-2010}, while the latter is
considered by the authors only as a second choice after PathSPCA.
All tested algorithms were implemented in Matlab 7.12 running on a
commodity desktop.

In this experiment, we regard the true model to be successfully recovered
when both quantities $|v_1^\top  u_1|$ and $|v_2^\top
u_2|$ are greater than $0.99$. We also assume that
the cardinality $k=10$ of the underlying sparse eigenvectors is
known a prior. Table~\ref{Tab:spca_toy_chance_success} lists the
recovering results by the tested methods. It can be observed that
TPower, PathPCA and GPower all successfully recover the ground truth
sparse PC vectors with extremely high rate of success. SPCA
frequently fails to recover the spares loadings on this dataset.
The potential reason is that SPCA is initialized with the ordinary
principal components which in many random data matrices are far away
from the truth sparse solution. Traditional PCA always fails to
recover the sparse PC loadings on this dataset.
The success of TPower and the failure of traditional PCA can be well explained by our sparse recovery result
in Theorem~\ref{thm:recovery}  (for TPower) in comparison to the traditional eigenvector perturbation theory
in Lemma~\ref{lem:perturb} (for traditional PCA), which we have already discussed in Section~\ref{sect:algorithm_analysis}.
However, the success of other methods suggests that
it might be possible to prove sparse recovery results similar to Theorem~\ref{thm:recovery} for some of these
alternative algorithms.

\begin{table}[htb]
\begin{center}
\caption{The quantitative results on a synthetic dataset. The
values $|u_1^\top u_2|$, $|v_1^\top u_1|$, $|v_2^\top u_2|$ are mean
of the 500 running. \label{Tab:spca_toy_chance_success}}
\begin{tabular}{ c  c c c c c}
\hline
Algorithms & Parameter & $|u_1^\top u_2|$ & $|v_1^\top u_1|$ & $|v_2^\top u_2|$ & Probability of success\\
\hline \hline
TPower & $k=10$ & 0 & 0.9998 & 0.9997 & 1\\
\hline
PathSPCA & $k=10$ & 0 & 0.9998 & 0.9997 & 1 \\
\hline
$\text{GPower}_{\ell_1,m}$ & $\gamma = 0.8$ & 0 & 0.9997 & 0.9996 & 0.99\\
\hline
$\text{GPower}_{\ell_0,m}$ & $\gamma = 0.8$ & $9.5 \times 10^{-5}$  & 0.9997 & 0.9991 & 0.99 \\
\hline
SPCA & $\lambda_1 = 10^{-3}$ & $2.4 \times 10^{-3}$ & 0.9274 & 0.9250 & 0.25\\
\hline
PCA & $ - $ & 0 & 0.9146 & 0.9086 & 0 \\
\hline
\end{tabular}
\end{center}
\end{table}

\subsubsection{Speed and Scaling Test}

To study the computational efficiency of TPower, we list in
Table~\ref{Tab:tpower_cpu_results} the CPU running time (in seconds)
by TPower on several datasets at different scales. The datasets are
generalized as $n \times p$ Gaussian random matrices with fixed
$n=500$, and exponentially increasing values of dimension $p$. We
set the termination criteria for TPower to be $\|Q(x_{t}) -
Q(x_{t-1})\| \le 10^{-4} $. It can be observed from
Table~\ref{Tab:tpower_cpu_results} that TPower can exit within
seconds or tens of seconds on all the datasets under a wider range of
cardinality $k$.

\begin{table}[htb]
\begin{center}
\caption{Average computational time for extracting one sparse
component (in seconds) by TPower under different setting of
dimensionality $p$ and cardinality $k$.
\label{Tab:tpower_cpu_results}}
\begin{tabular}{ c  c c c c c c}
\hline
$n \times p$ & $500\times 1000$ & $500 \times 2000$ & $500 \times 4000$ & $500 \times 8000$ & $500 \times 16000$ & $500 \times 32000$ \\
\hline \hline
$k=0.05 \times p$  & 0.07 & 0.13 & 0.39 & 0.88 & 2.52 & 7.03 \\
\hline
$k=0.1 \times p$  & 0.07 & 0.14 & 0.55 & 1.00 & 3.31 & 20.05 \\
\hline
$k=0.2 \times p$  & 0.13 & 0.21 & 0.96 & 2.00 & 7.94 & 21.78 \\
\hline
$k=0.5 \times p$  & 0.12 & 0.43 & 2.06 & 5.08 & 19.63 & 25.19 \\
\hline
\end{tabular}
\end{center}
\end{table}

\subsubsection{Results on \textsf{PitProps} Data}

The \textsf{Pitprops} dataset~\citep{Jefeers-1967}, which consists
of 180 observations with 13 measured variables, has been a standard
benchmark to evaluate algorithms for sparse PCA~\citep[See,
e.g.,][]{Zou-2006,Shen-2008,Journee-2010}. Following these previous
studies, we also consider to compute the first six sparse PCs of the
data. In Table~\ref{Tab:spca_pitprops_results}, we list the total
cardinality and the proportion of adjusted variance~\citep{Zou-2006}
explained by six components computed with TPower,
PathSPCA~\citep{Aspremont-PathSPCA-2008},
GPower~\citep{Journee-2010} and SPCA~\citep{Zou-2006}. From these
results we can see that on this relatively simple dataset, TPower,
PathSPCA and GPower perform quite similarly. SPCA is inferior to the
other three algorithms.

Table~\ref{Tab:spca_pitprops_pcs} lists the six extracted PCs by
TPower with cardinality setting 7-2-1-1-1-1. We can see that the
important variables associated with the six principal components do
not overlap, which leads to a clear interpretation of the extracted
components. The same loadings are extracted by both PathSPCA and
GPower under the parameters listed in
Table~\ref{Tab:spca_pitprops_results}.

\begin{table}[htb]
\begin{center}
\caption{The quantitative results on the \textsf{PitProps} dataset.
The result of SPCA is taken from~\citep{Zou-2006}.
\label{Tab:spca_pitprops_results}}
\begin{tabular}{ c  c c c}
\hline Method & Parameters & Total cardinality & Prop. of explained variance\\
\hline \hline
TPower & cardinalities: 8-8-4-2-2-2 & 26 & 0.8636 \\
TPower & cardinalities: 7-2-3-1-1-1 & 15 & 0.8230 \\
TPower & cardinalities: 7-2-1-1-1-1 & 13 & 0.7599 \\
\hline
PathSPCA & cardinalities: 8-8-4-2-2-2 & 26 & 0.8615 \\
PathSPCA & cardinalities: 7-2-3-1-1-1 & 15 & 0.8230 \\
PathSPCA & cardinalities: 7-2-1-1-1-1 & 13 & 0.7599 \\
\hline
$\text{GPower}_{\ell_1,m}$ & $\gamma = 0.22$ & 26 & 0.8438 \\
$\text{GPower}_{\ell_1,m}$ & $\gamma = 0.30$ & 15 & 0.8230 \\
$\text{GPower}_{\ell_1,m}$ & $\gamma = 0.40$ & 13 & 0.7599 \\
\hline
SPCA & see~\citep{Zou-2006} & 18 & 0.7580 \\
\hline
\end{tabular}
\end{center}
\end{table}

\begin{table}[htb]
\center \scriptsize
\begin{center}
\caption{The extracted six PCs on \textsf{PitProps} dataset by
TPower with cardinality setting 7-2-1-1-1-1. Note that in this
setting, the extracted six PCs are orthogonal and the significant
loadings are non-overlapping. \label{Tab:spca_pitprops_pcs}}
\begin{tabular}{ c c c c c c c c c c c c c c}
\hline
PCs & \tabincell{c}{$x_1$ \\ topd} & \tabincell{c}{$x_2$ \\ length}  & \tabincell{c}{$x_3$ \\ moist}  & \tabincell{c}{$x_4$ \\ testsg} & \tabincell{c}{$x_5$ \\ ovensg}  & \tabincell{c}{$x_6$ \\ ringt}  & \tabincell{c}{$x_7$ \\ ringb} & \tabincell{c}{$x_8$ \\ bowm}  & \tabincell{c}{$x_9$ \\ bowd}  & \tabincell{c}{$x_{10}$ \\ whorls} & \tabincell{c}{$x_{11}$ \\ clear}  & \tabincell{c}{$x_{12}$ \\ knots}  & \tabincell{c}{$x_{13}$ \\ diaknot} \\
\hline \hline
PC1 & 0.4235 & 0.4302 & 0 & 0& 0 & 0.2680 & 0.4032 & 0.3134 & 0.3787 & 0.3994 & 0 & 0 & 0  \\
\hline
PC2 & 0 & 0 & 0.7071 & 0.7071 & 0 & 0 & 0 & 0 & 0 & 0 & 0 & 0 & 0  \\
\hline
PC3 & 0 & 0 & 0 & 0 & 1.000 & 0 & 0 & 0 & 0 & 0 & 0 & 0 & 0  \\
\hline
PC4 & 0 & 0 & 0 & 0 & 0 & 0 & 0 & 0 & 0 & 0 & 1.000 & 0 & 0  \\
\hline
PC5 & 0 & 0 & 0 & 0 & 0 & 0 & 0 &  & 0 & 0 & 0 & 1.000 & 0  \\
\hline
PC6 & 0 & 0 & 0 & 0 & 0 & 0 & 0 & 0 & 0 & 0 & 0 & 0 & 1.000  \\
\hline
\end{tabular}
\end{center}
\end{table}

\subsubsection{Results on Biological Data}

We have also evaluated the performance of TPower on two gene
expression datasets, one is the \textsf{Colon cancer} data
from~\citep{Alon-1999}, the other is the \textsf{Lymphoma} data
from~\citep{Alizadeh-2000}. Following the experimental setup in
\citep{Aspremont-PathSPCA-2008}, we consider the $500$ genes with
the largest variances. We plot the variance versus cardinality
tradeoff curves in Figure~\ref{fig:exp_path_biological}, together
with the result from PathSPCA~\citep{Aspremont-PathSPCA-2008} and
the upper bounds of optimal values
from~\citep{Aspremont-PathSPCA-2008}. Note that our method performs
almost identical to the PathSPCA which is demonstrated to have
optimal or very close to optimal solutions in many cardinalities.
The computational time of the two methods on both datasets is
comparable and is less than two seconds.

\begin{figure}
\centering
\subfigure[\textsf{Colon
Cancer}]{\includegraphics[width=80mm]{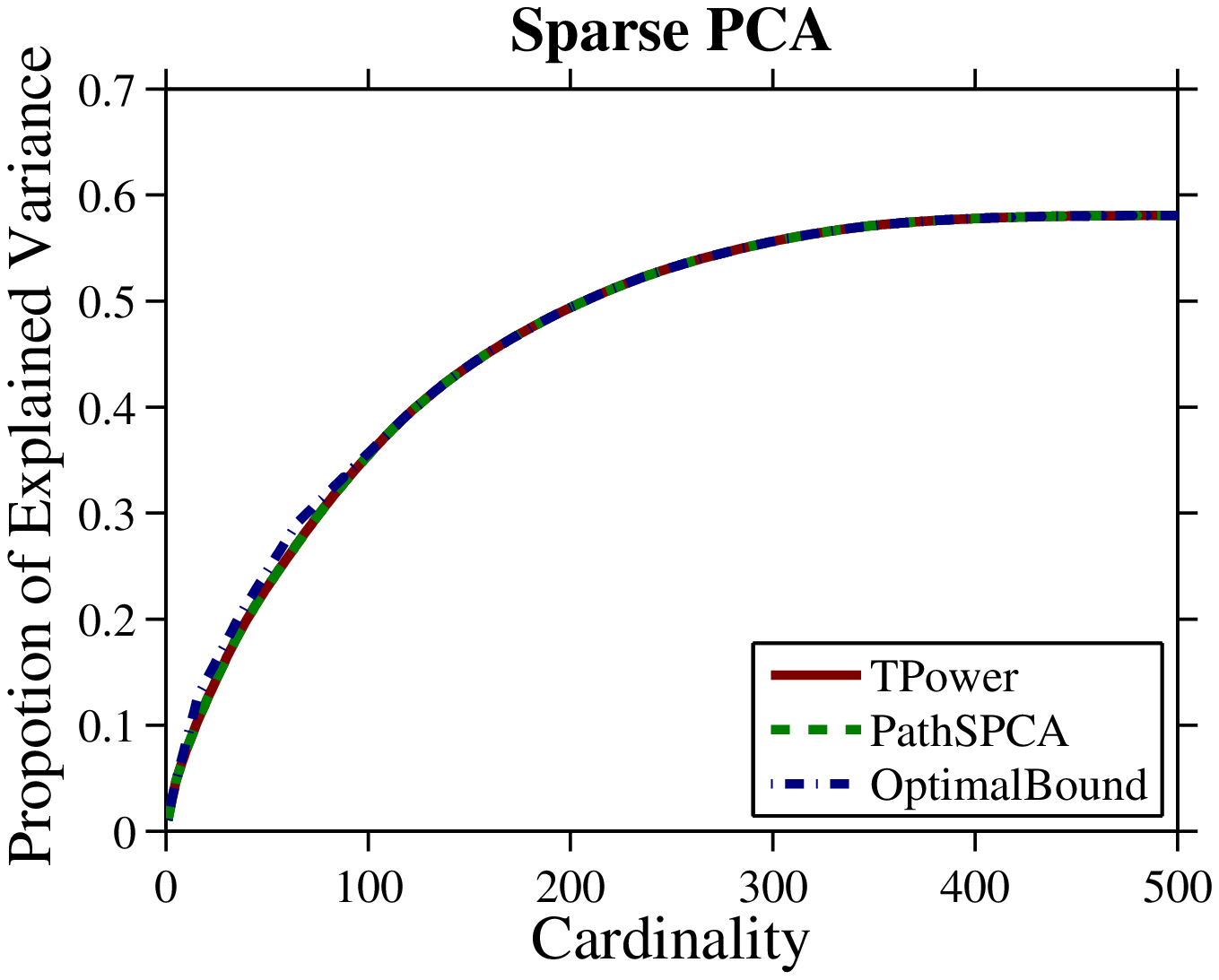}\label{fig:colon_cancer}}
\subfigure[\textsf{Lymphoma}]{\includegraphics[width=80mm]{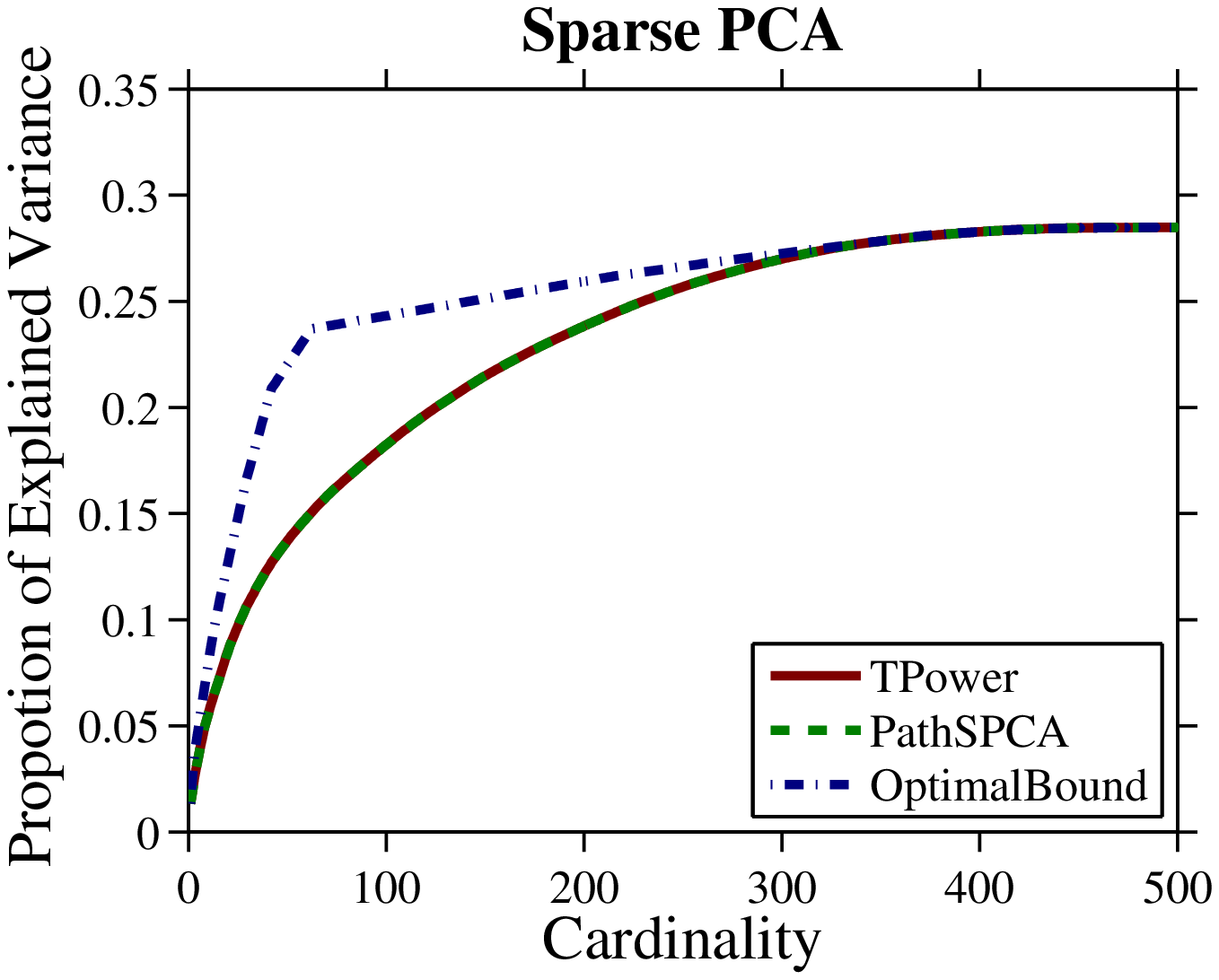}\label{fig:toy_outlier_removal}}
\caption{The synthetic dataset and outlier removal result.  For
better viewing, please see the original pdf file.
}\label{fig:exp_path_biological}
\end{figure}

\subsubsection{Results on Document Data}

In this section we evaluate the practical performance of TPower for
key terms extraction on a document dataset\textsf{ 20 Newsgroups}
(\textsf{20NG}). The
\textsf{20NG}~\footnote{\url{http://people.csail.mit.edu/jrennie/20Newsgroups/}}
is a dataset collected and originally used for document
classification by~\citet{Lang-1995}. A total number of $18,846$
documents, evenly distributed across 20 classes, are left after
removing duplicates and newsgroup-identifying headers. This corpus
contains $26,214$ distinct terms after stemming and stop word
removal. Each document is then represented as a term-frequency
vector and normalized to one. We use the top 1,000 terms according
to the DF (document frequency) of the terms in the corpus. We
extract 5 sparse PCs on this dataset. The cardinality setting for
the 5 sparse PCs is 20-20-10-10-10. Table~\ref{Tab:spca_ng20} lists
the terms associated with the 1st, 2nd and 5th sparse PCs. The
interpretation is quite clear: the 1st sparse PC is about figures,
the 2nd is about computer science, and the 5th is on religion. We
have observed that quite similar terms are extracted by PathSPCA
under the same cardinality setting. Here we do not list the 3rd and
4th PCs since they overlap with the listed ones due to the
non-orthogonality of sparse PCs.

\begin{table}[htb]
\begin{center}
\caption{Terms associated with the first 1st, 2nd, 5th sparse PCs.
Dataset: \textsf{20NG}. \label{Tab:spca_ng20}}
\begin{tabular}{ c c c }
\hline
1st PC (20 terms) &  2nd PC (20 terms)  & 5th PC (10 terms) \\
\hline \hline \hline
``1''& ``edu'' & ``dont''  \\

``2'' & ``system'' & ``time'' \\

``3'' & ``inform'' & ``people''   \\

``4'' & ``includ'' & ``believ''  \\

``5'' & ``support'' & ``god''  \\

``10'' & ``program'' & ``exist''   \\

``20'' & ``version'' & ``christan''  \\

``15'' & ``set'' & ``jesus''   \\

``8'' &``window'' & ``christ''   \\

``6''& ``softwar'' & ``atheist''   \\

``16''& ``avail''  &    \\

``12''& ``file''  &  \\

``14''& ``data''  &  \\

``7''& ``user''  &   \\

``18''& ``grafic''  &   \\

``9''& ``color''  &   \\

``13'' & ``imag''  &   \\

``11'' & ``displai''  &   \\

``0''& ``format''  &   \\

``la''& ``ftp''  &   \\
\hline
\end{tabular}
\end{center}
\end{table}

\subsubsection{Summary}

To summarize this group of experiments on sparse PCA, the basic
finding is that TPower performs quite competitively in terms of the
trade-off between explained variance and representation sparsity.
The performance is comparable to
PathSPCA~\citep{Aspremont-PathSPCA-2008} and
GPower~\citep{Journee-2010} both on the synthetic and on the real
datasets. It is observed that TPower, PathSPCA and GPower outperform
SPCA~\citep{Zou-2006} on the benchmark data \textsf{Pitprops}.
Although performing quite similarly, TPower, PathSPCA and GPower are
different algorithms: TPower is a power iteration method while
PathSPCA is a greedy forward selection method, both directly address
the cardinality constrained sparse eigenvalue
problem~\eqref{prob:spase_eigenvalue}, while GPower is a power
iteration method for certain regularized versions of sparse
eigenvalue problem (see the previous
Section~\ref{ssect:connections}). While strong theoretical guarantee
can be established for the TPower method, it remains open to show
that PathSPCA and GPower have a similar sparse recovery performance.

\subsection{Densest $k$-Subgraph Finding}
\label{ssect:dense_k_subgraph}

As another concrete application, we show that with proper
modification, TPower can be applied to the densest $k$-subgraph
finding problem. Given an undirected graph $G=(V,E)$, $|V|=n$, and
integer $1 \le k \le n$, the densest $k$-subgraph (DkS) problem is
to find a set of $k$ vertices with maximum average degree in the
subgraph induced by this set. In the weighted version of DkS we are
also given nonnegative weights on the edges and the goal is to find
a $k$-vertex induced subgraph of maximum average edge weight.
Algorithms for finding DkS are useful tools for analyzing networks.
In particular, they have been used to select features for
ranking~\citep{Geng-2007}, to identify cores of
communities~\citep{KUMAR-1999}, and to combat link
spam~\citep{GIBSON-2005}.

It has been shown that the DkS problem is NP hard for bipartite
graphs and chordal graphs~\citep{Corneil-1984}, and even for graphs
of maximum degree three~\citep{Feige-2001}. A large body of
algorithms have been proposed based on a variety of techniques
including greedy
algorithms~\citep{Feige-2001,Asahiro-2002,Ravi-1994}, linear
programming~\citep{Billionnet-2004,Khuller-2009}, and semidefinite
programming~\citep{Srivastav-1998,Ye-2003}. For general $k$, the
algorithm developed by \citet{Feige-2001} achieves the best
approximation ratio of $O(n^\epsilon)$ where $\epsilon < 1/3$.
\citet{Ravi-1994} proposed 4-approximation algorithms for weighted
DkS on complete graphs for which the weights satisfy the triangle
inequality. \citet{Liazi-2008} has presented a 3-approximation
algorithm for DkS for chordal graphs. Recently, \citet{Jiang-2010}
proposed to reformulate DkS as a 1-mean clustering problem and
developed a $2$-approximation to the reformulated clustering
problem. Moreover, based on this reformulation, \citet{Yang-2010}
proposed a $1+\epsilon$-approximation algorithm with certain
exhaustive (and thus expensive) initialization procedure. In
general, however, \citet{Khot-2006} showed that DkS has no
polynomial time approximation scheme (PTAS), assuming that there are
no sub-exponential time algorithms for problems in NP.

Mathematically, DkS can be restated as the following binary
quadratic programming problem:
\begin{equation}\label{prob:DkS}
\max_{\pi \in \mathbb{R}^n} \pi^\top W\pi, \ \ \ \ \ \ \
\text{subject to } \pi \in \{1,0\}^n, \|\pi\|_0 = k,
\end{equation}
where $W$ is the (non-negative weighted) adjacency matrix of $G$. If
$G$ is an undirected graph, then $W$ is symmetric. If $G$ is
directed, then $A$ could be asymmetric. In this latter case, from
the fact that $\pi^\top  W \pi = \pi^\top  \frac{W + W^\top }{2}
\pi$, we may equivalently solve Problem~\eqref{prob:DkS} by
replacing $W$ with $\frac{W + W^\top }{2}$. Therefore, in the
following discussion, we always assume that the affinity matrix $W$
is symmetric (or $G$ is undirected).

\subsubsection{The TPower-DkS Algorithm}
We propose the TPower-DkS algorithm as a slight modification of
TPower, to solve the DkS problem. The process generates a sequence
of intermediate vectors $\pi_0,\pi_1,...$ from a starting vector
$\pi_0$. At each step $t$ the vector $\pi_{t-1}$ is multiplied by
the matrix $W$, then $\pi_t$ is set to be the indicator vector of
the top $k$ entries in $W\pi_{t-1}$. The TPower-Dks is formally
given in Algorithm~\ref{alg:sparse_power_dks}.


\begin{algorithm}[h]
\SetKwInOut{Input}{Input} \SetKwInOut{Output}{Output}
\SetKwInOut{Parameters}{Parameters}
\SetKwInOut{Initialization}{Initialization}

\Input{$W \in \mathbb{S}^{n}_{+}$,, initial vector $\pi_0 \in \mathbb{R}^n$}
\Output{$\pi_t$}
\Parameters{cardinality $k \in \{1,...,n\}$}
\BlankLine Let $t=1$.

\Repeat{Convergence} {

Compute $\pi'_t = W\pi_{t-1}$.

Identify $F_{t} = \supp(\pi'_t,k)$ the index set of $\pi'_t$ with
top $k$ values.

Set $\pi_t$ to be 1 on the index set $F_t$, and 0 otherwise.

$t \leftarrow t + 1$.

 }

\caption{Truncated Power Method for DkS
(TPower-DkS)}\label{alg:sparse_power_dks}
\label{alg:sparse_power_eig}
\end{algorithm}

\begin{remark}
By relaxing the constraint $\pi \in \{0,1\}^n$ to
$\|\pi\|=\sqrt{k}$, we may convert the densest $k$-subgraph
problem~\eqref{prob:DkS} to the standard sparse eigenvalue
problem~\eqref{prob:spase_eigenvalue} (up to a scaling) and then
directly apply TPower (in Algorithm 1) for solution. Our numerical
experience shows that such a relaxation strategy also works
satisfactory in practice, although is slightly inferior to
TPower-DkS (in Algorithm 2) which directly addresses the original
problem.
\end{remark}

Note that in Algorithm~\ref{alg:sparse_power_dks} we require that
$W$ is positive semidefinite. The motivation of this requirement is
to guarantee the convexity of the objective in
problem~\eqref{prob:DkS}, and thus following the similar arguments
in~\citep{Journee-2010} it can be shown that the objective value
will be monotonically increasing during the iterations. In many
real-world DkS problems, however, it is often the case that the
affinity matrix $W$ is not positive semi-definite. In this case, the
objective is non-convex and thus the monotonicity of TPower-DkS does
not hold. However, this complication can be circumvented by instead
running the algorithm with the shifted quadratic function:
\[
\max_{\pi \in \mathbb{R}^n}\pi^\top  (W + \tilde{\lambda} I_{p
\times p})\pi, \ \ \ \text{subject to } \pi \in \{0,1\}^n, \|\pi\|_0
= k.
\]
where $\tilde\lambda>0$ is large enough such that $\tilde{W}=W +
\tilde\lambda I_{p\times p} \in \mathbb{S}^n_{+}$. On the domain of
interest, this change only adds a constant term to the objective
function. The TPower-DkS, however, produces a different sequence of
iterates, and there is a clear trade-off. If the second term
dominates the first term (say by choosing a very large
$\tilde\lambda$), the objective function becomes approximately a
squared norm, and the algorithm tends to terminate in very few
iterations. In the limiting case of $\tilde\lambda \to \infty$, the
method will not move away from the initial iterate. To handle this
issue, we propose to gradually increase $\tilde\lambda$ during the
iterations and we do so only when the monotonicity is violated. To
be precise, if at a time instance $t$, $\pi_t^\top  W \pi_t <
\pi_{t-1}^\top  W \pi_{t-1}$, then we add $\tilde\lambda I_{p\times
p}$ to $W$ with a gradually increased $\tilde\lambda$ by repeating
the current iteration with the updated matrix until $\pi_t^\top  (W
+ \tilde\lambda I_{p\times p}) \pi_t \ge \pi_{t-1}^\top (W + \lambda
I_{p\times p}) \pi_{t-1}$~\footnote{Note that the inequality
$\pi_t^\top (W + \tilde\lambda I_{p\times p}) \pi_t \ge
\pi_{t-1}^\top (W + \tilde\lambda I_{p\times p}) \pi_{t-1}$ is
deemed to be satisfied when $\tilde\lambda$ is large enough, e.g.,
when $W+\tilde\lambda I_{p\times p} \in \mathbb{S}^n_+$.}, which
implies $\pi_t^\top W \pi_t \ge \pi_{t-1}^\top W \pi_{t-1}$.

\subsubsection{On Initialization}

Since TPower-DkS is a monotonically increasing procedure, it
guarantees to improve the initial point $\pi_0$. Basically, any
existing approximation DkS method, e.g., greedy
algorithms~\citep{Feige-2001,Ravi-1994}, can be used to initialize
TPower-DkS. In our numerical experiments, we observe that by simply
setting $\pi_{0}$ as the indicator vector of the vertices with the
top $k$ (weighted) degrees, our method can achieve very competitive
results on all the real-world datasets we have tested on.

\subsubsection{Results on Web Graphs}

We have tested TPower on four page-level web graphs:
\textsf{cnr-2000}, \textsf{amazon-2008}, \textsf{ljournal-2008},
\textsf{hollywood-2009}, from the WebGraph framework provided by the
Laboratory for Web Algorithms~\footnote{Datasets are available
at~\url{http://lae.dsi.unimi.it/datasets.php}}. We treated each
directed arc as an undirected edge.
Table~\ref{Tab:statistics_webgraph} lists the statistics of the
datasets used in the experiment.

\begin{table}[htb]
\begin{center}
\caption{The statistics of the web graph datasets.
\label{Tab:statistics_webgraph}}
\begin{tabular}{ c  c c c}
\hline Graph & Nodes ($|V|$) & Total Arcs ($|E|$) & Average Degree\\
\hline \hline
\textsf{cnr-2000} & 325,557 & 3,216,152 & 9.88 \\
\hline
\textsf{amazon-2008} & 735,323 & 5,158,388 & 7.02 \\
\hline
\textsf{ljournal-2008} & 5,363,260 & 79,023,142 & 14.73 \\
\hline
\textsf{hollywood-2009} & 1,139,905 & 113,891,327 & 99.91 \\
\hline
\end{tabular}
\end{center}
\end{table}

We compare our TPower-DkS method with two greedy methods for the DkS
problem. One greedy method is proposed by \citet{Ravi-1994} which is
referred to as Greedy-Ravi in our experiments. The Greedy-Ravi
algorithm works as follows: it starts from a heaviest edge and
repeatedly adds a vertex to the current subgraph to maximize the
weight of the resulting new subgraph; this process is repeated until
$k$ vertices are chosen. The other greedy method is developed
by~\citet[Procedure 2]{Feige-2001} which is referred as Greedy-Feige
in our experiments. The procedure works as follows: let $S$ denote
the $k/2$ vertices with the highest degrees in $G$; let $C$ denote
the $k/2$ vertices in the remaining vertices with largest number of
neighbors in $S$; return $S \cup C$.

Figure~\ref{fig:web_graph} shows the density value $\pi^\top W
\pi/k$ and CPU time versus the cardinality $k$. From the density
curves we can observe that on \textsf{cnr-2000},
\textsf{ljournal-2008} and \textsf{hollywood-2009}, TPower-DkS
consistently outputs denser subgraphs than the two greedy
algorithms, while on \textsf{amazon-2008}, TPower-DkS and
Greedy-Ravi are comparable and both are better than Greedy-Feige.
For CPU running time, it can be seen from the right column of
Figure~\ref{fig:web_graph} that Greedy-Feige is the fastest among
the three methods while TPower-DkS is only slightly slower. This is
due to the fact that TPower-DkS needs iterative matrix-vector
products while Greedy-Feige only needs a few degree sorting outputs.
Although TPower-DkS is slightly slower than Greedy-Feige, it is
still quite efficient. For example, on \textsf{hollywood-2009} which
has hundreds of millions of arcs, for each $k$, Greedy-Feige
terminates within about 1 second while TPower terminates within
about 10 seconds. The Greedy-Ravi method is however much slower than
the other two on all the graphs when $k$ is large.

\begin{figure}
\centering
\subfigure[cnr-2000]{\includegraphics[width=50mm]{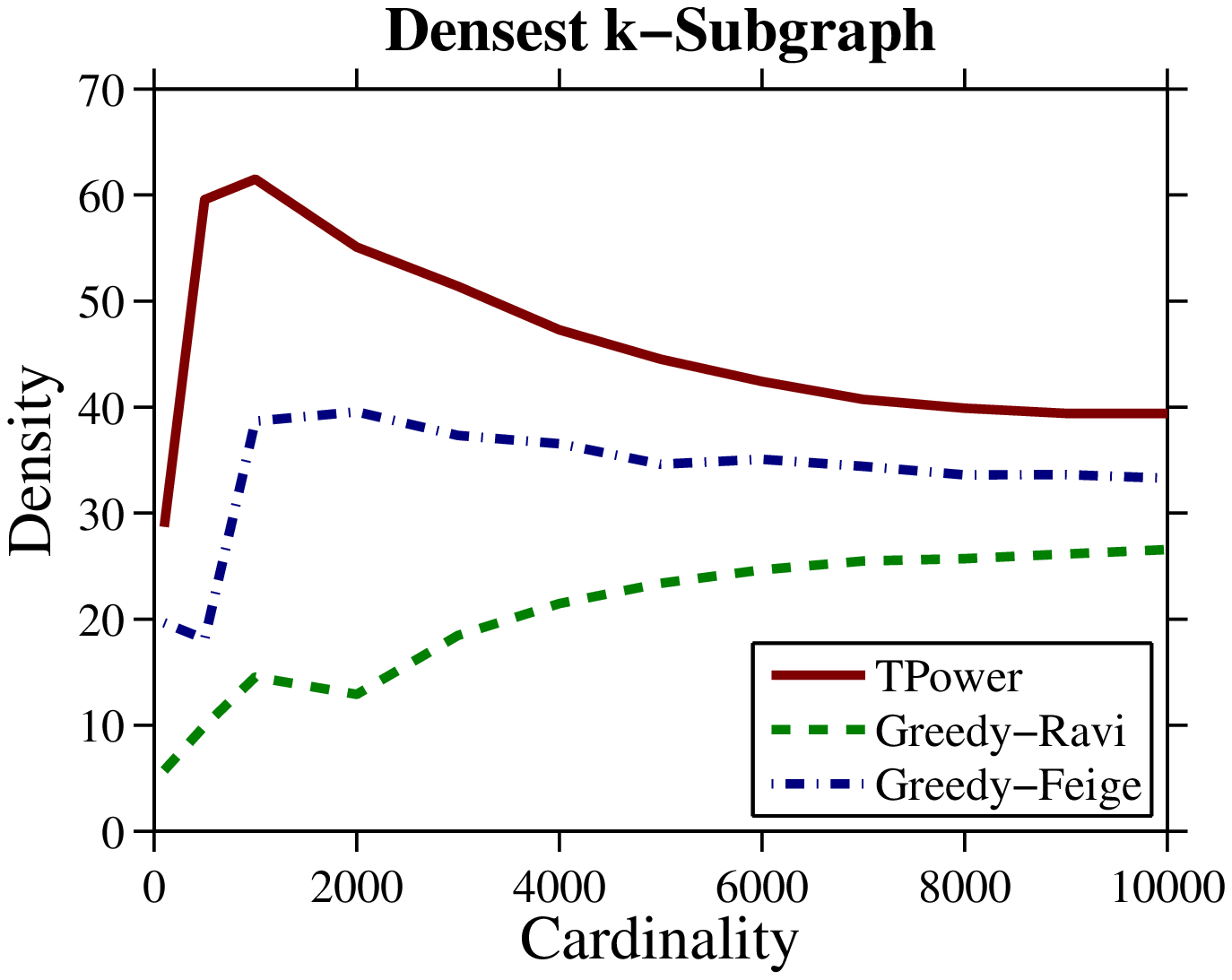}\label{fig:cnr_2000_density}
\includegraphics[width=50mm]{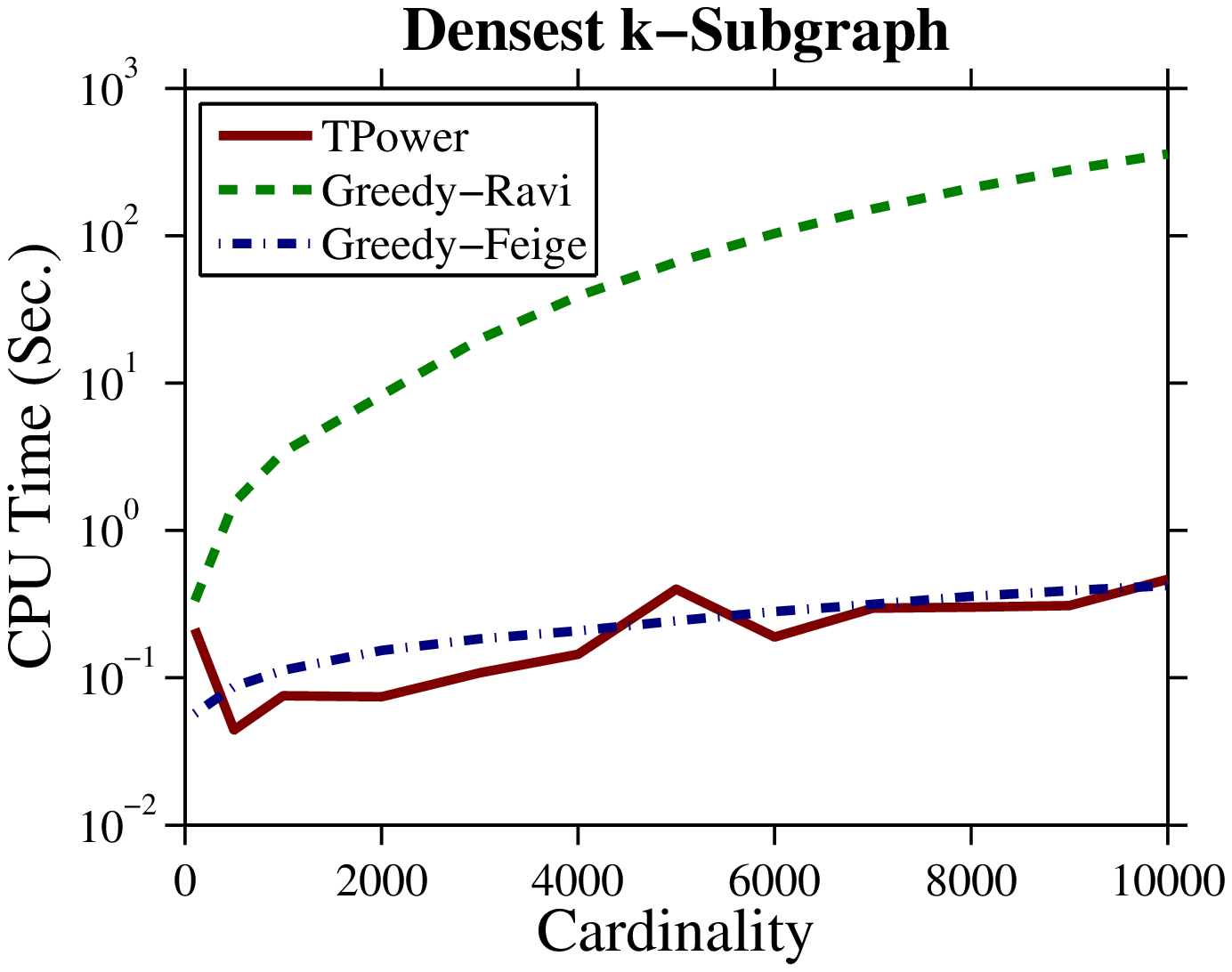}\label{fig:cnr_2000_cpu}}

\subfigure[amazon-2008]{\includegraphics[width=50mm]{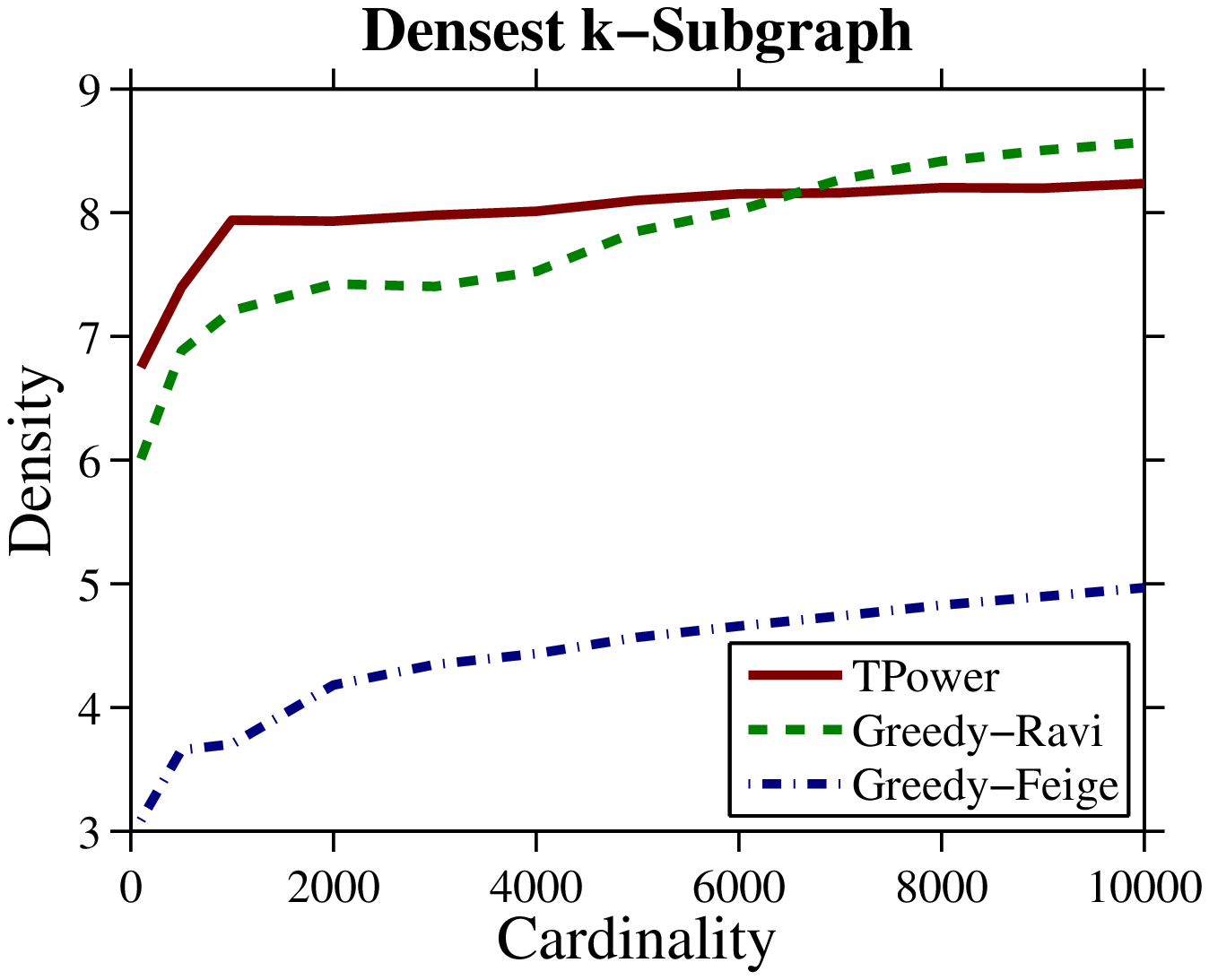}\label{fig:amazon_2008_density}
\includegraphics[width=50mm]{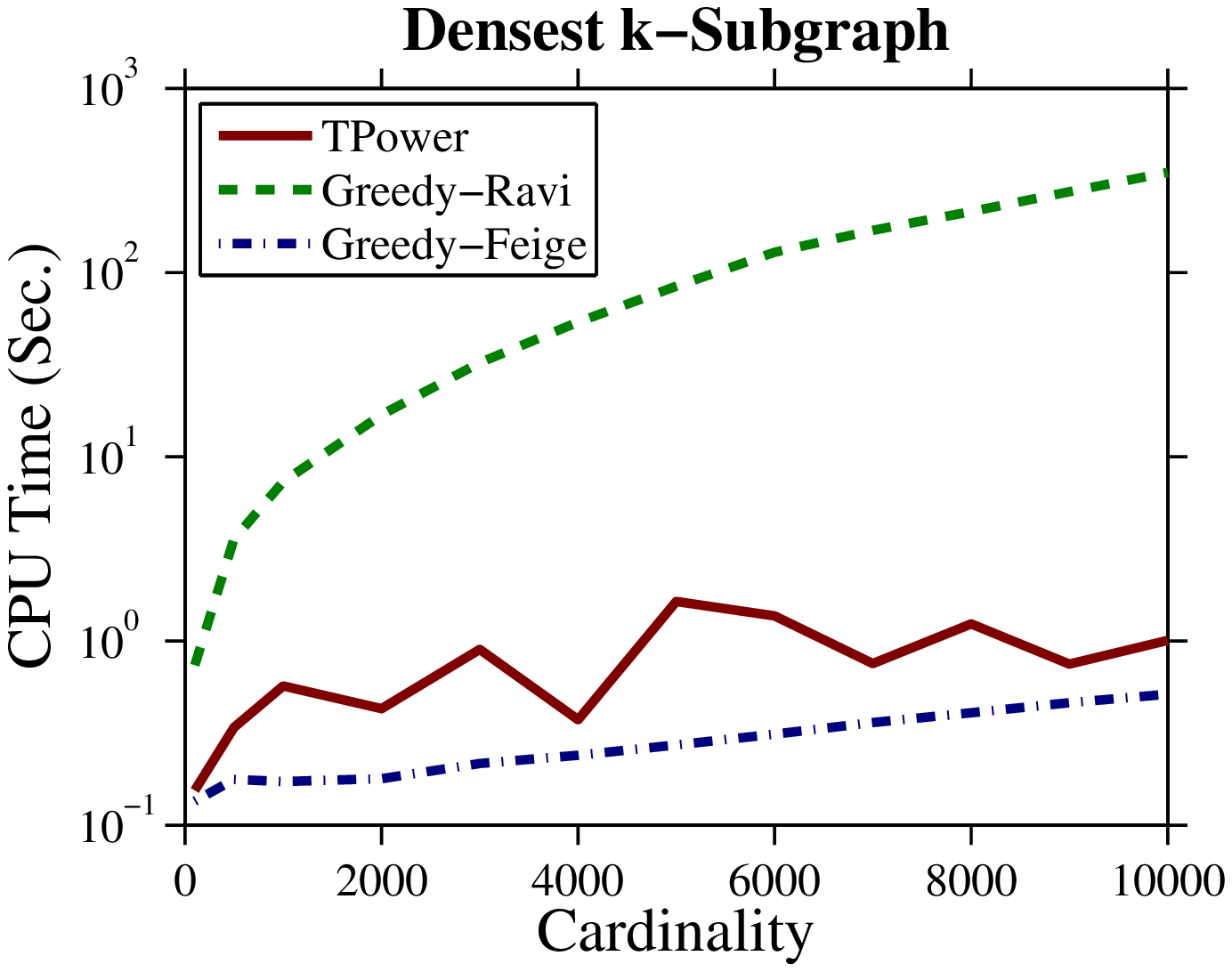}\label{fig:amazon_2008_cpu}}

\subfigure[ljournal-2008]{\includegraphics[width=50mm]{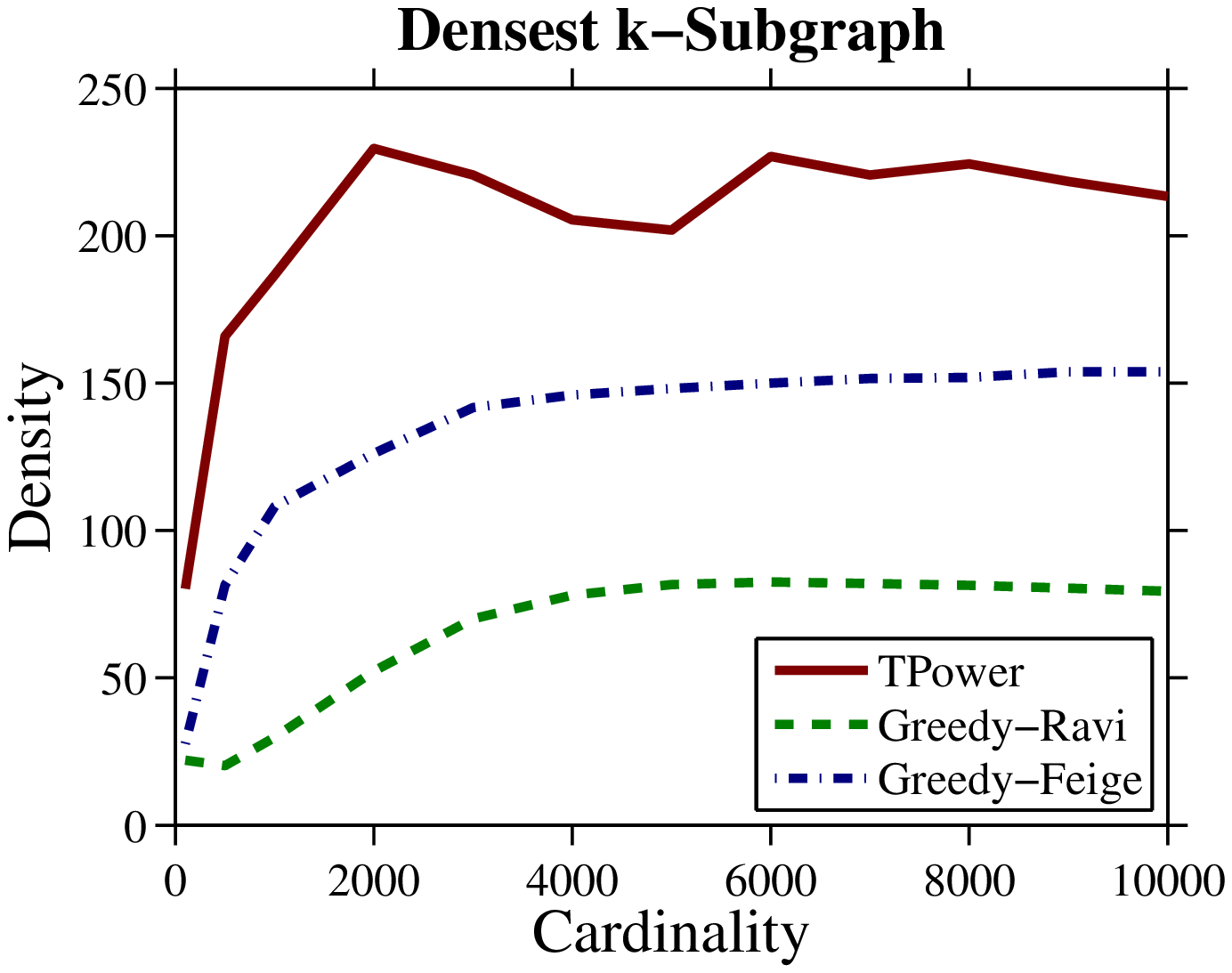}\label{fig:ljournal_2008_density}
\includegraphics[width=50mm]{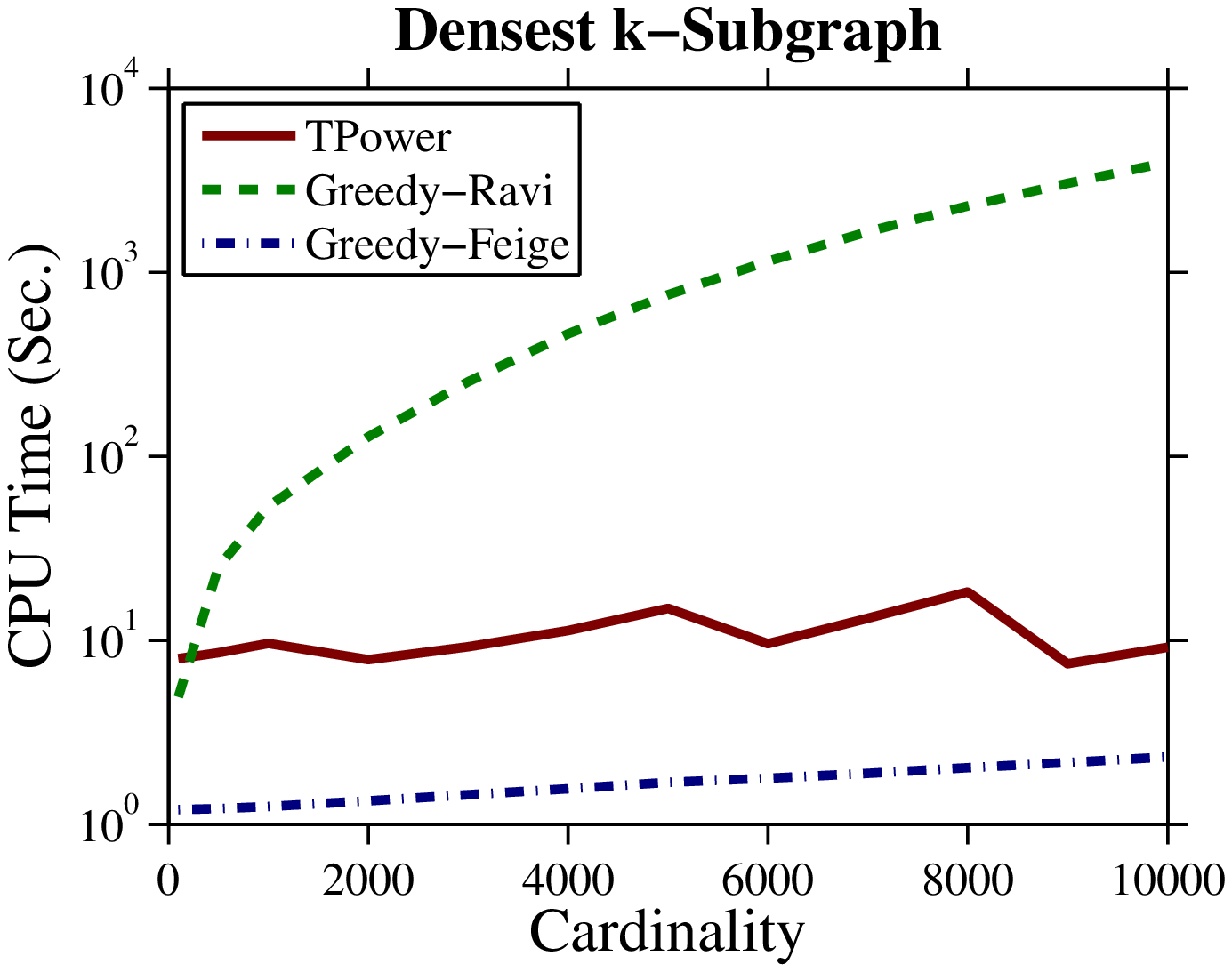}\label{fig:ljournal_2008_cpu}}

\subfigure[hollywood-2009]{\includegraphics[width=50mm]{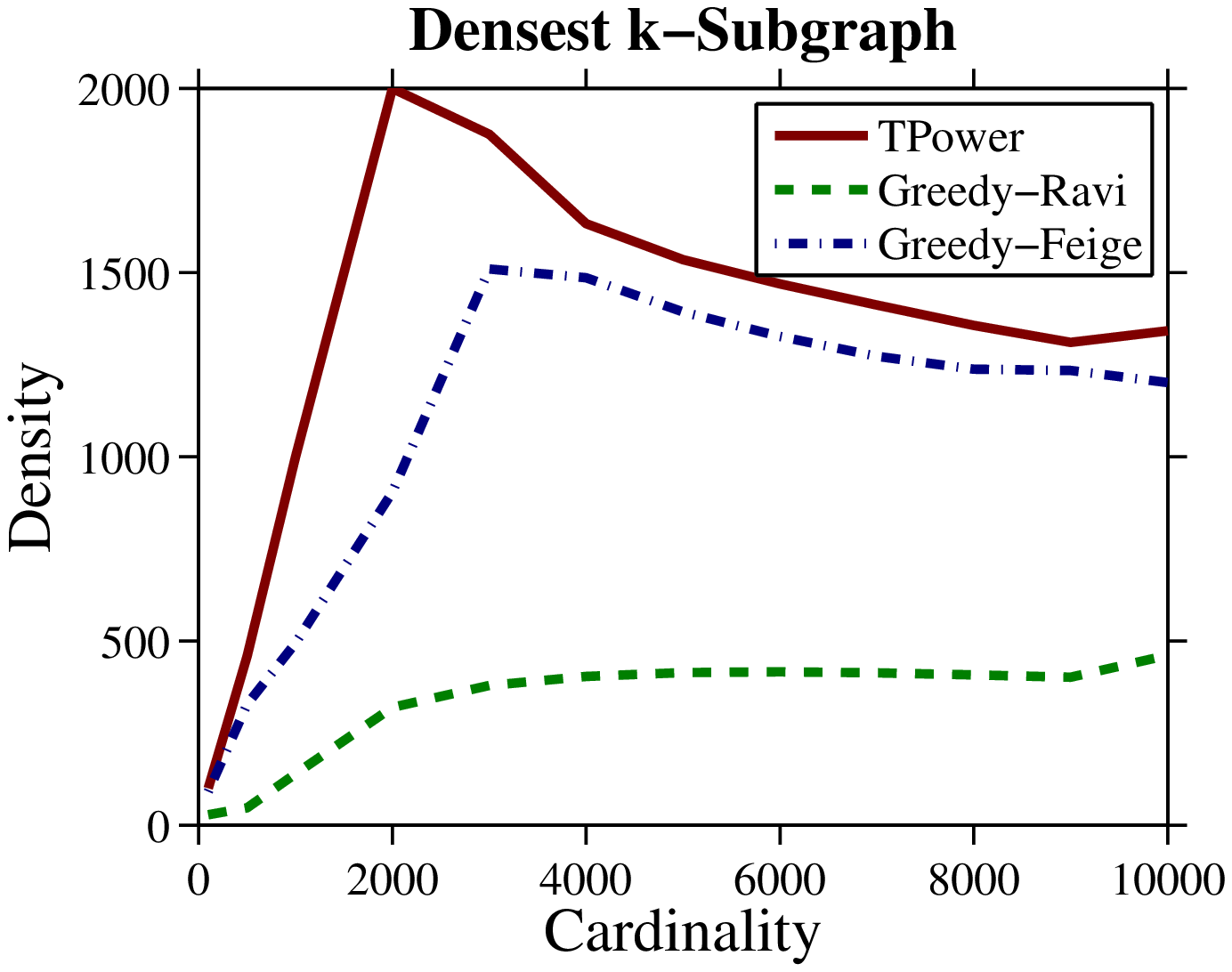}\label{fig:hollywood_2009_density}
\includegraphics[width=50mm]{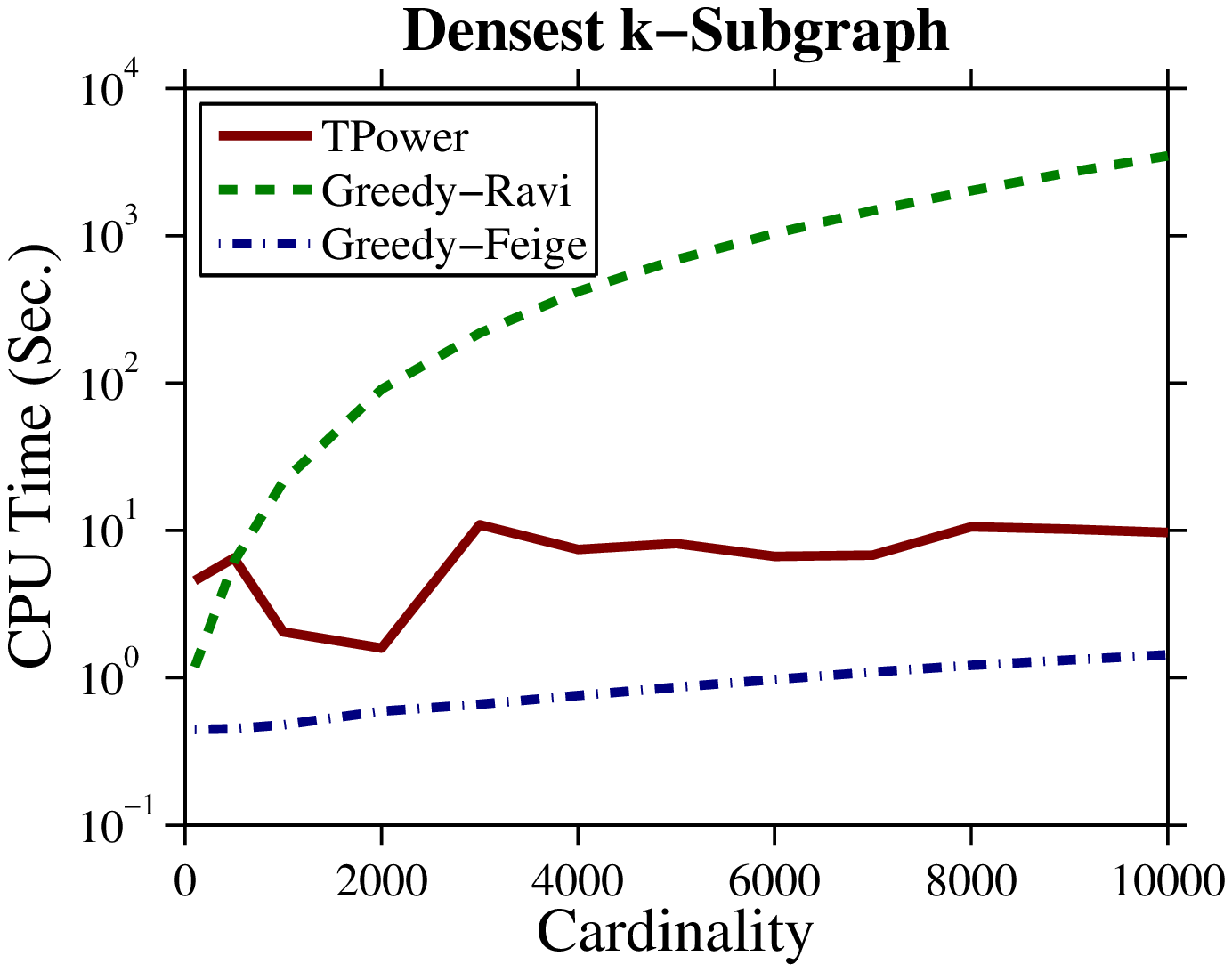}\label{fig:hollywood_2009_cpu}}

\caption{Identifying densest $k$-subgraph on four web graphs. Left:
density curves as a function of $k$. Right: CPU time curves as a
function of $k$. }\label{fig:web_graph}
\end{figure}

\subsubsection{Results on Air-Travel Routine}

We have applied TPower-DkS to identify subsets of American and
Canadian cities that are most easily connected to each other, in
terms of estimated commercial airline travel time. The
graph~\footnote{The data is available
at~\url{www.psi.toronto.edu/affinitypropogation}} is of size
$|V|=456$ and $|E|=71,959$: the vertices are $456$ busiest
commercial airports in United States and Canada, while the weight
$w_{ij}$ of edge $e_{ij}$ is set to the inverse of the mean time it
takes to travel from city $i$ to city $j$ by airline, including
estimated stopover delays. Due to the headwind effect, the transit
time can depend on the direction of travel; thus $36\%$ of the
weight are asymmetric. Figure~\ref{fig:route_map_all} shows a map of
air-travel routine.

As in the previous experiment, we compare TPower-DkS to Greedy-Ravi
and Greedy-Feige on this dataset. For all the three algorithms, the
densities of $k$-subgraphs under different $k$ values are shown in
Figure~\ref{fig:route_density_comparison}, and the CPU running time
curves are given in Figure~\ref{fig:route_cpu_comparison}. From the
former figure we observe that TPower-DkS consistently outperforms
the other two greedy algorithms in terms of the density of the
extracted $k$-subgraphs. From the latter figure we can see that
TPower-DkS is slightly slower than Greed-Feige but much faster than
Greedy-Ravi.
Figure~\ref{fig:route_hks_30_spapo},\ref{fig:route_hks_30_ravi}, and
\ref{fig:route_hks_30_feige} illustrate the densest $k$-subgraph
with $k=30$ outputted by the three algorithms. In each of these
three subgraph, the red dot indicates the representing city with the
largest (weighted) degree. Both TPower-DkS and Greedy-Feige reveal
30 cities in east US. The former takes \emph{Cleveland} as the
representing city while the latter \emph{Cincinnati}. Greedy-Ravi
reveals 30 cities in west US and CA and takes \emph{Vancouver} as
the representing city. Visual inspection shows that the subgraph
recovered by TPower-DkS is the densest among the three.

After discovering the densest $k$-subgraph, we can eliminate their
nodes and edges from the graph and then apply the algorithms on the
reduced graph to search for the next densest subgraph. Such a
sequential procedure can be repeated to find multiple densest
$k$-subgraphs.
Figure~\ref{fig:route_hks_30x6_spapo},\ref{fig:route_hks_30x6_ravi},
and \ref{fig:route_hks_30x6_feige} illustrate sequentially estimated
six densest $30$-subgraphs by the three algorithms. Again, visual
inspection shows that our method output more geographically compact
subsets of cities than the other two. As a quantitative result, the
total density of the six subgraphs discovered by the three
algorithms is: 1.14 (TPower-DkS), 0.90 (Greedy-Feige) and 0.99
(Greedy-Ravi), respectively.

\begin{figure}
\centering
\subfigure[The air-travel route
map]{\includegraphics[width=50mm]{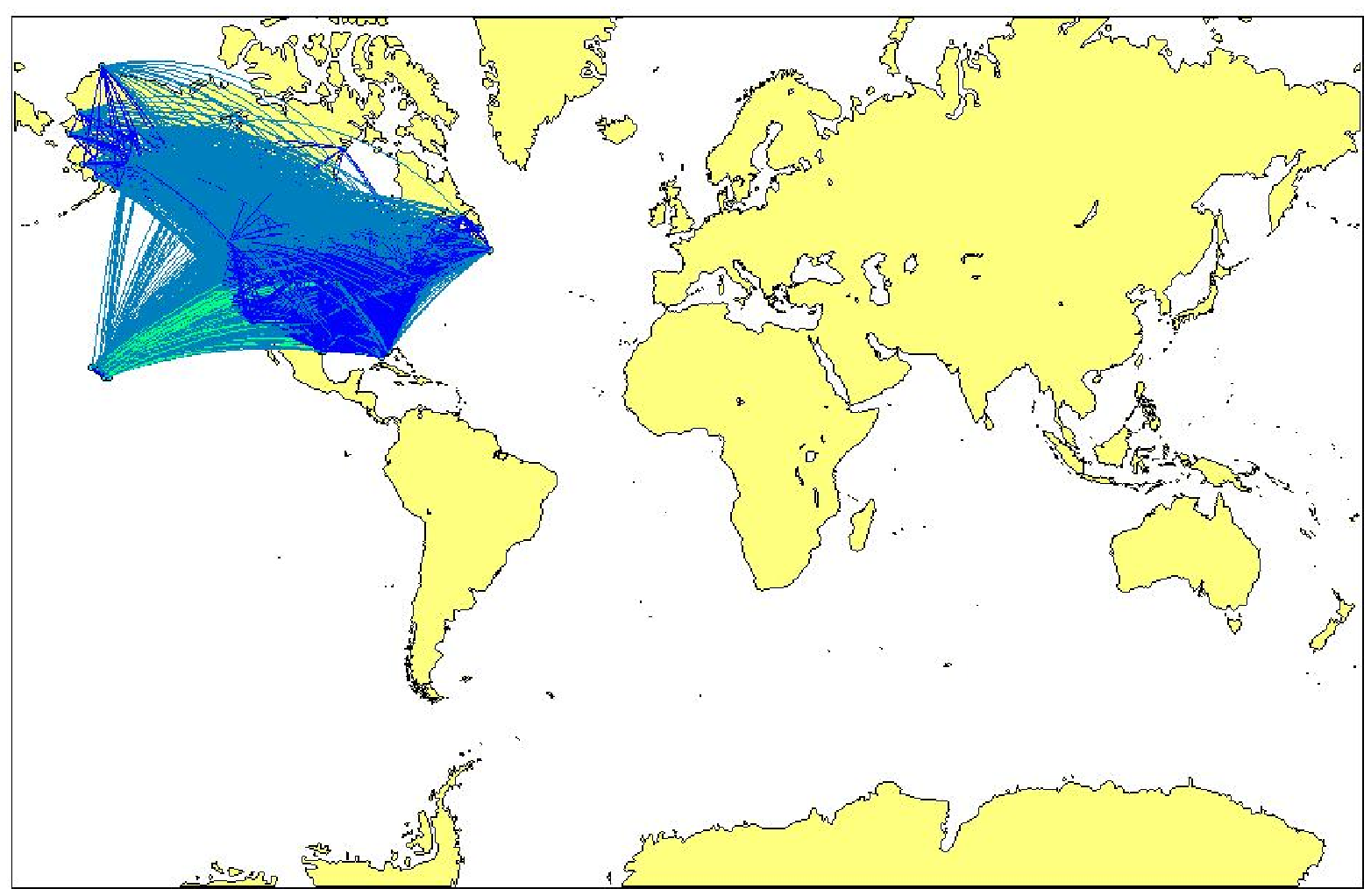}\label{fig:route_map_all}}
\subfigure[Density]{\includegraphics[width=50mm]{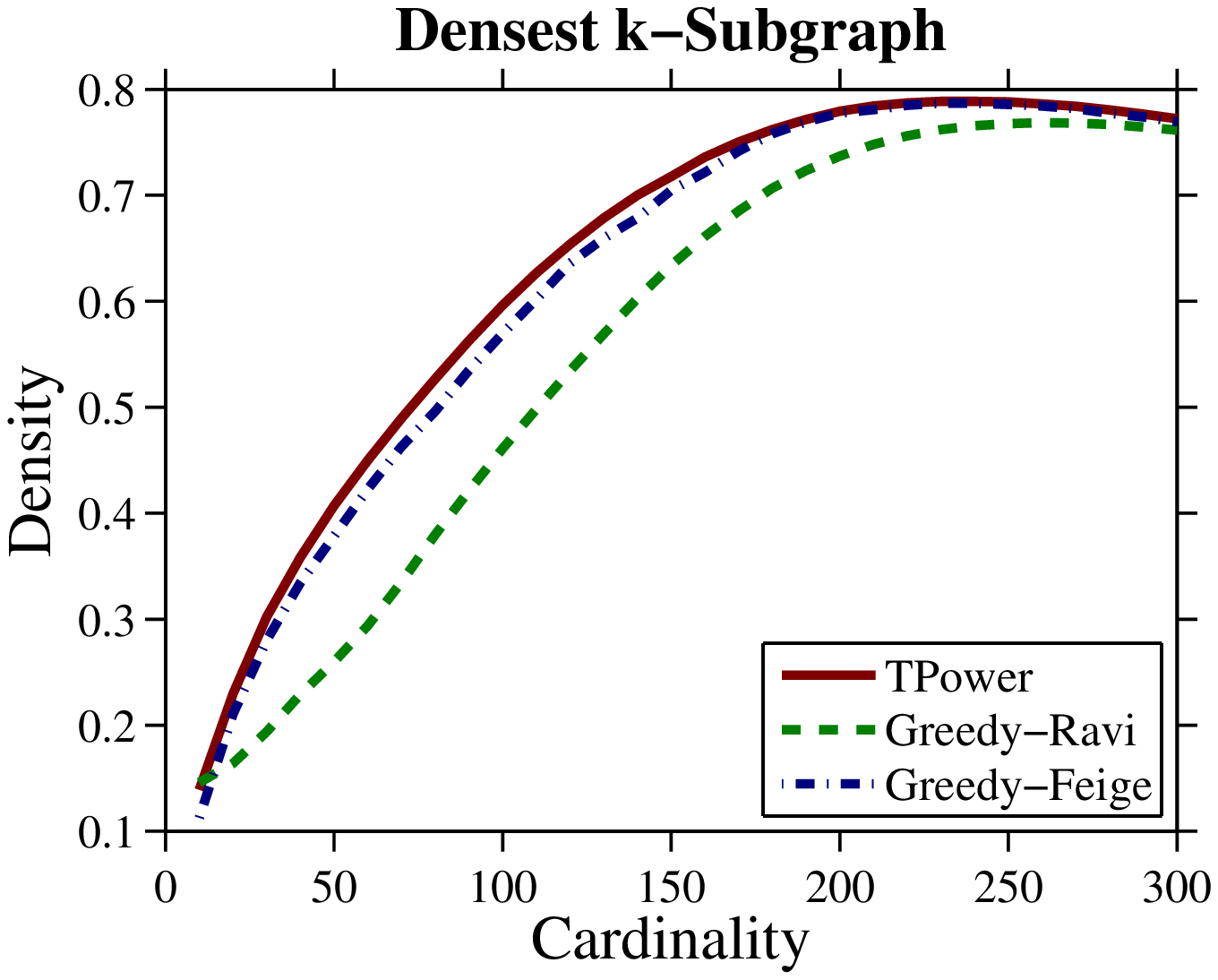}\label{fig:route_density_comparison}}
\subfigure[CPU
Time]{\includegraphics[width=50mm]{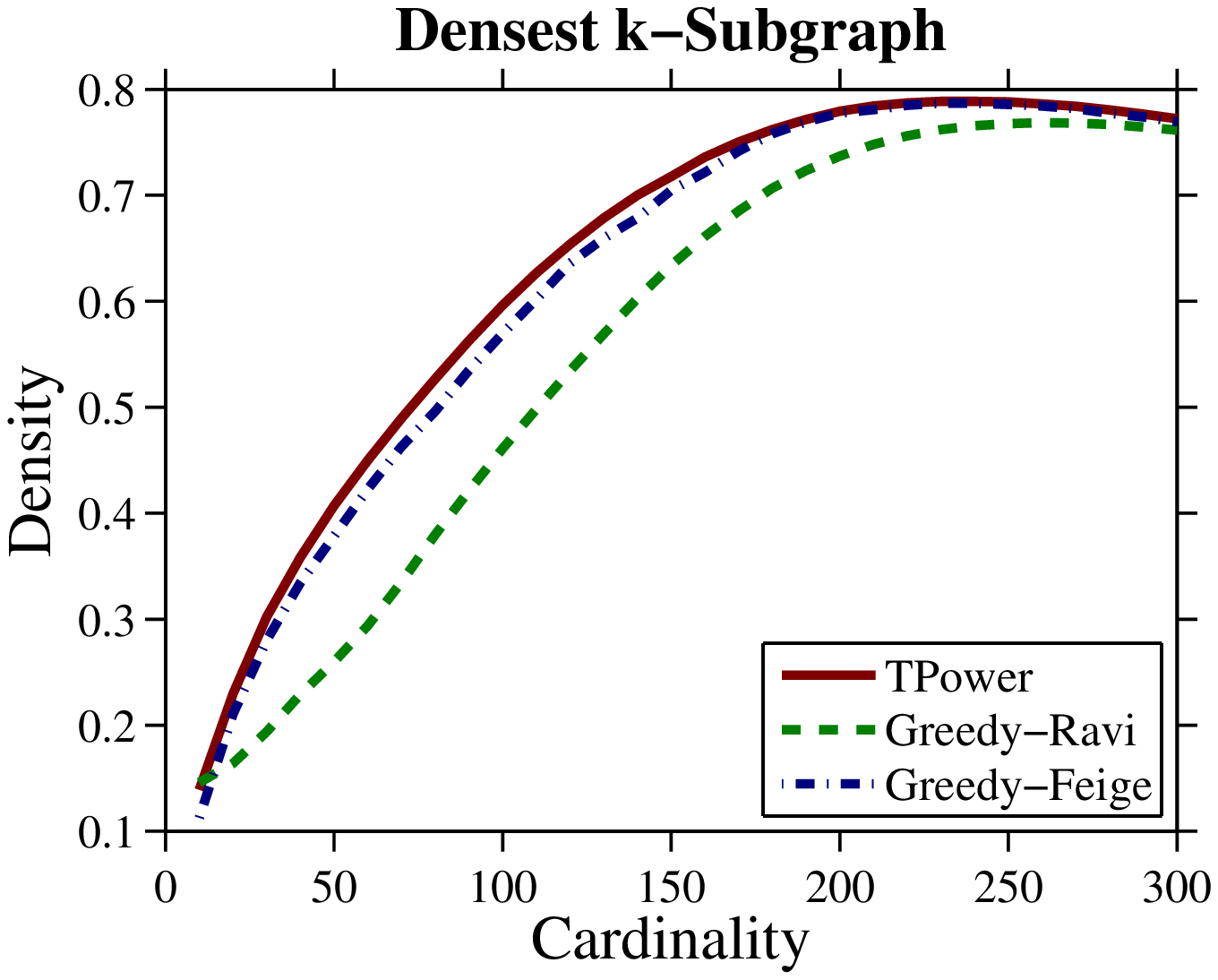}\label{fig:route_cpu_comparison}}

\subfigure[TPower-DkS]{\includegraphics[width=50mm]{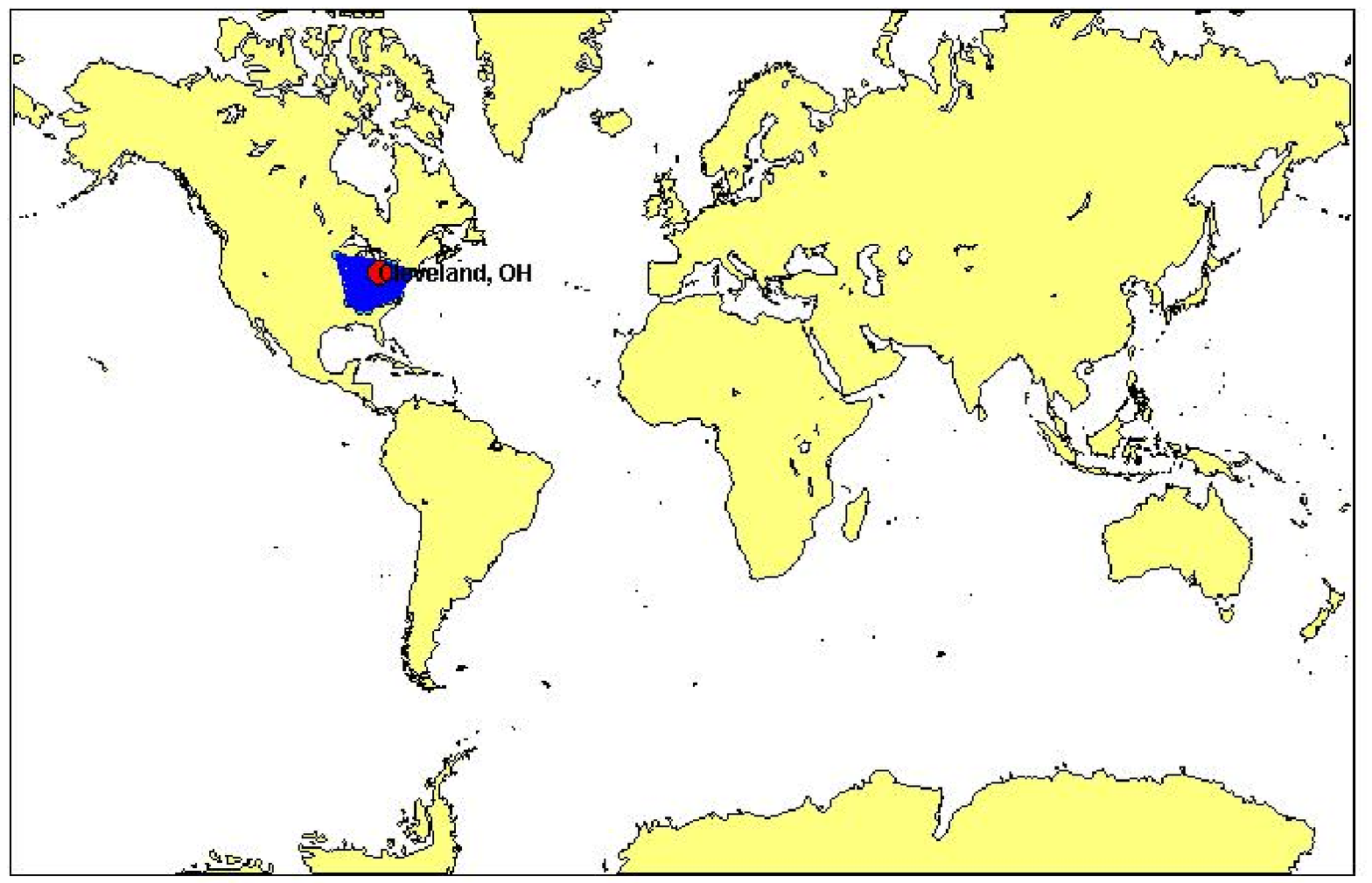}\label{fig:route_hks_30_spapo}}
\subfigure[Greedy-Ravi]{\includegraphics[width=50mm]{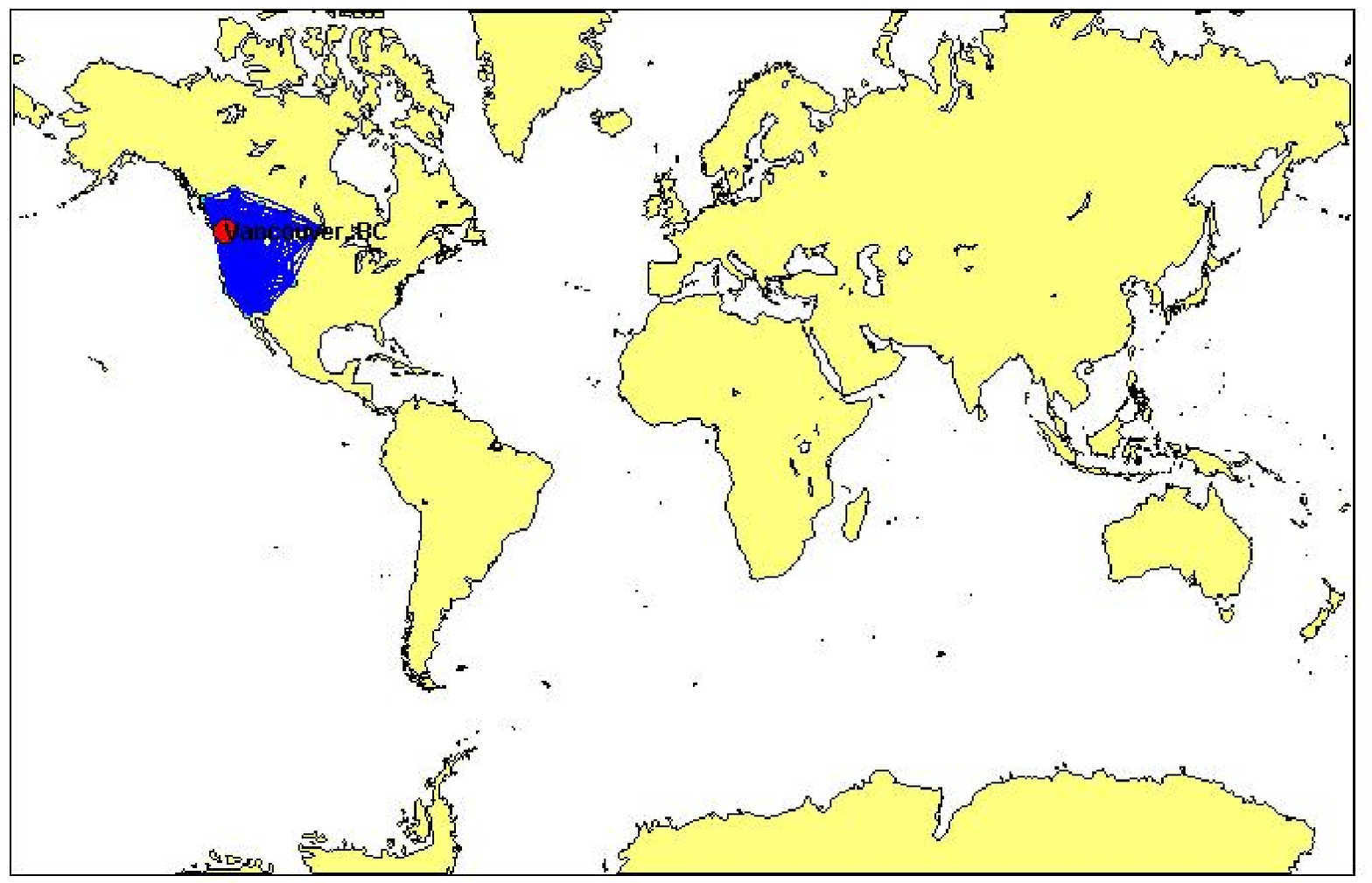}\label{fig:route_hks_30_ravi}}
\subfigure[Greedy-Feige]{\includegraphics[width=50mm]{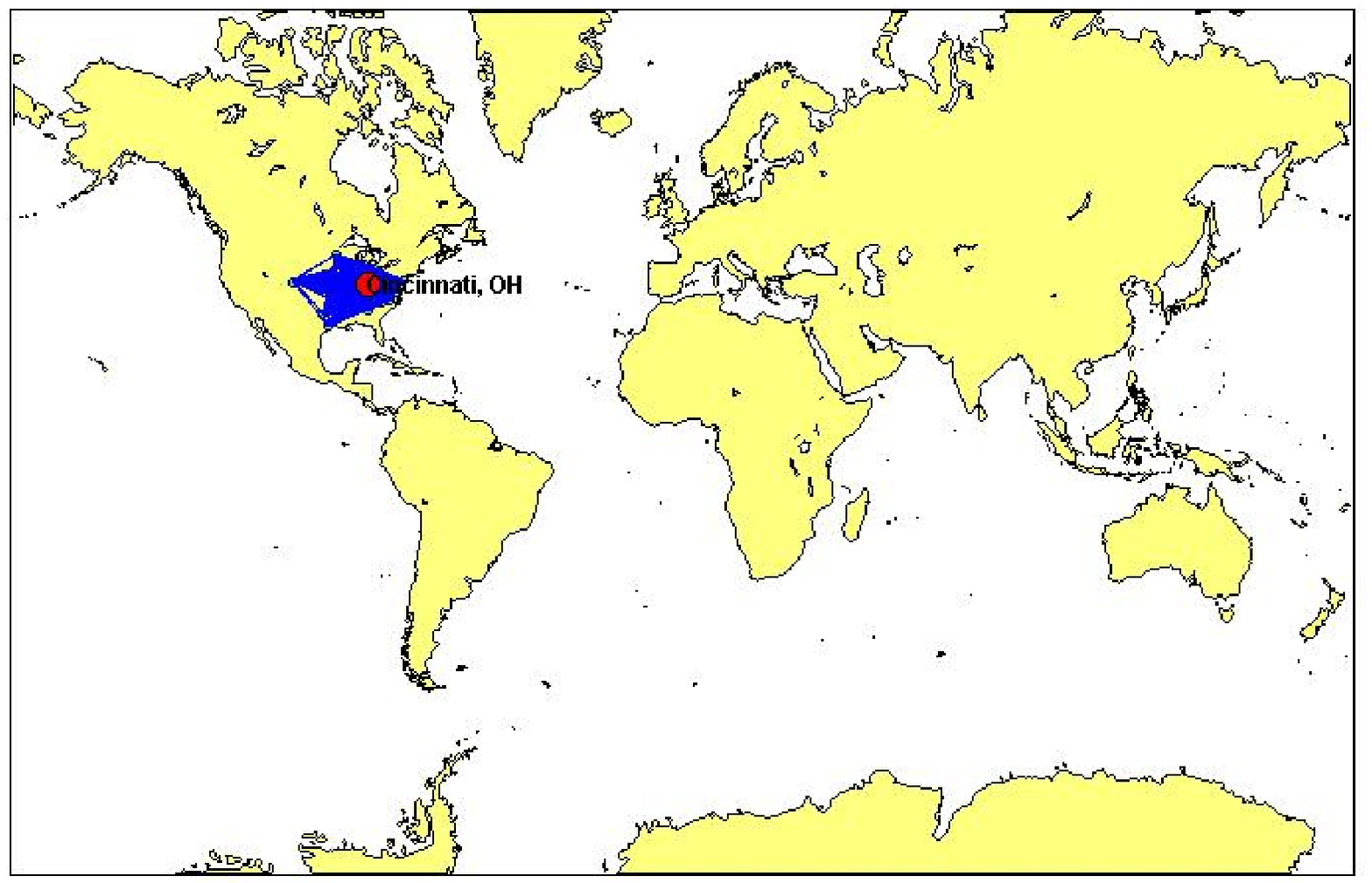}\label{fig:route_hks_30_feige}}

\subfigure[TPower-DkS]{\includegraphics[width=50mm]{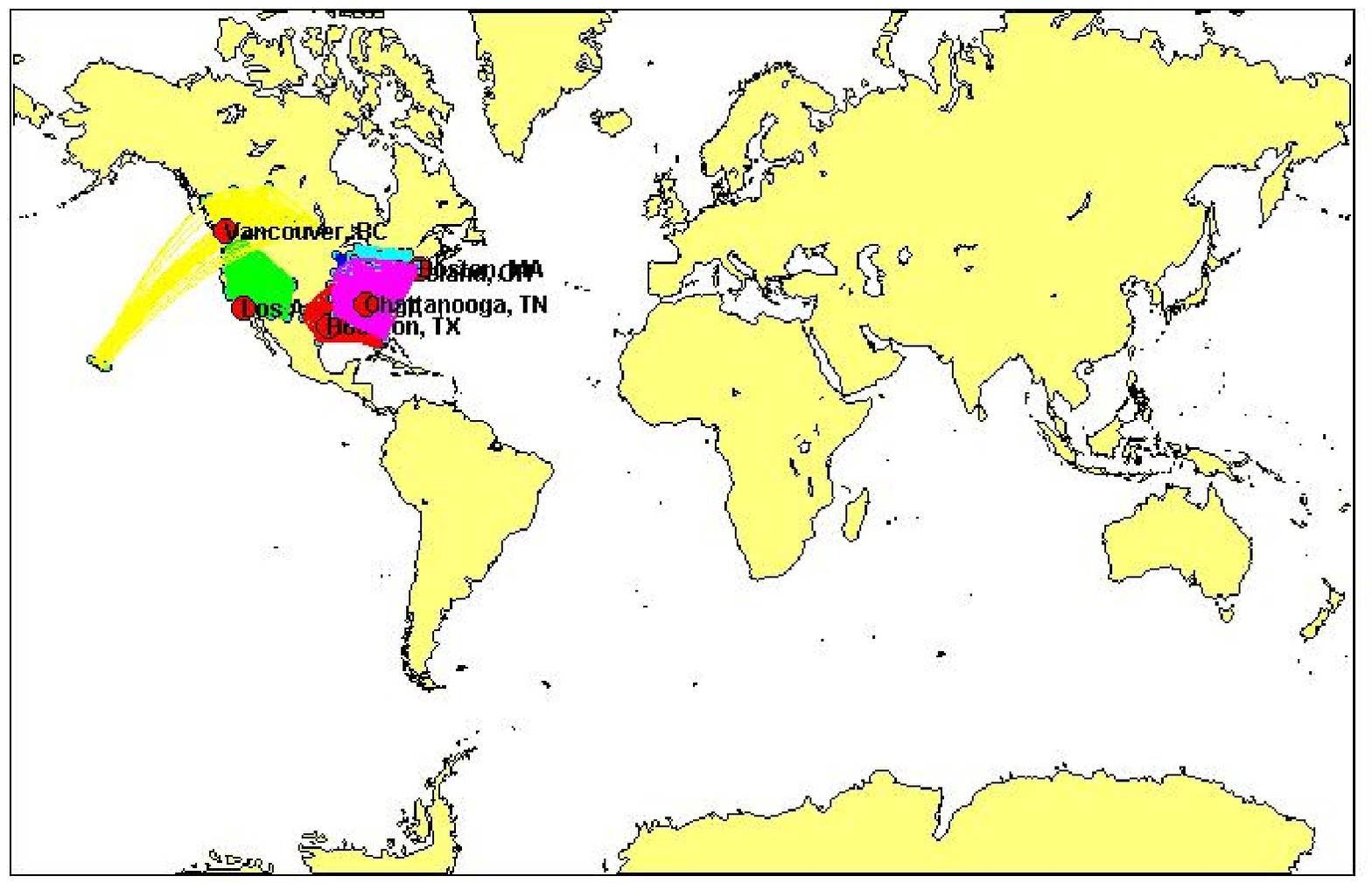}\label{fig:route_hks_30x6_spapo}}
\subfigure[Greedy-Ravi]{\includegraphics[width=50mm]{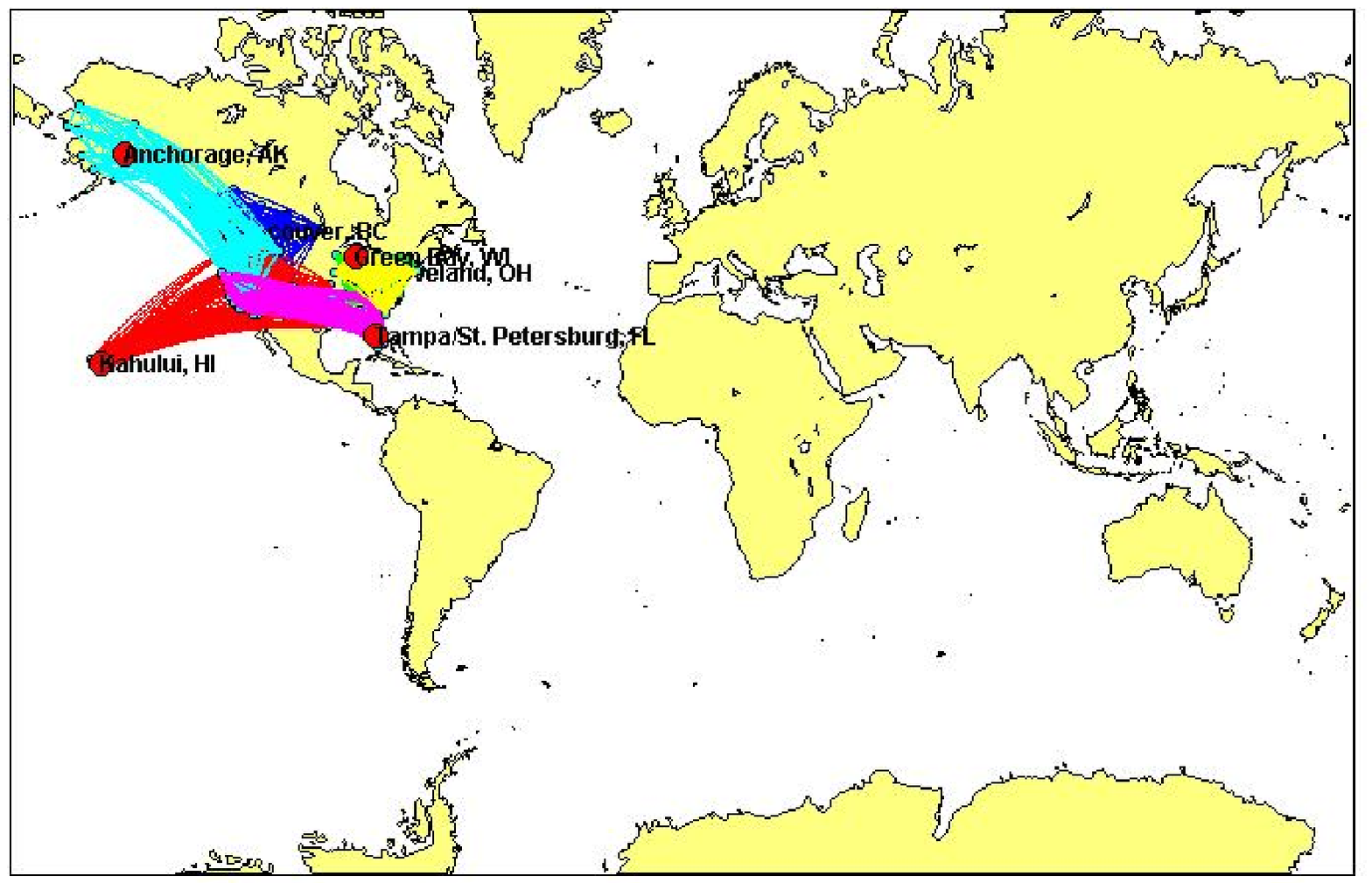}\label{fig:route_hks_30x6_ravi}}
\subfigure[Greedy-Feige]{\includegraphics[width=50mm]{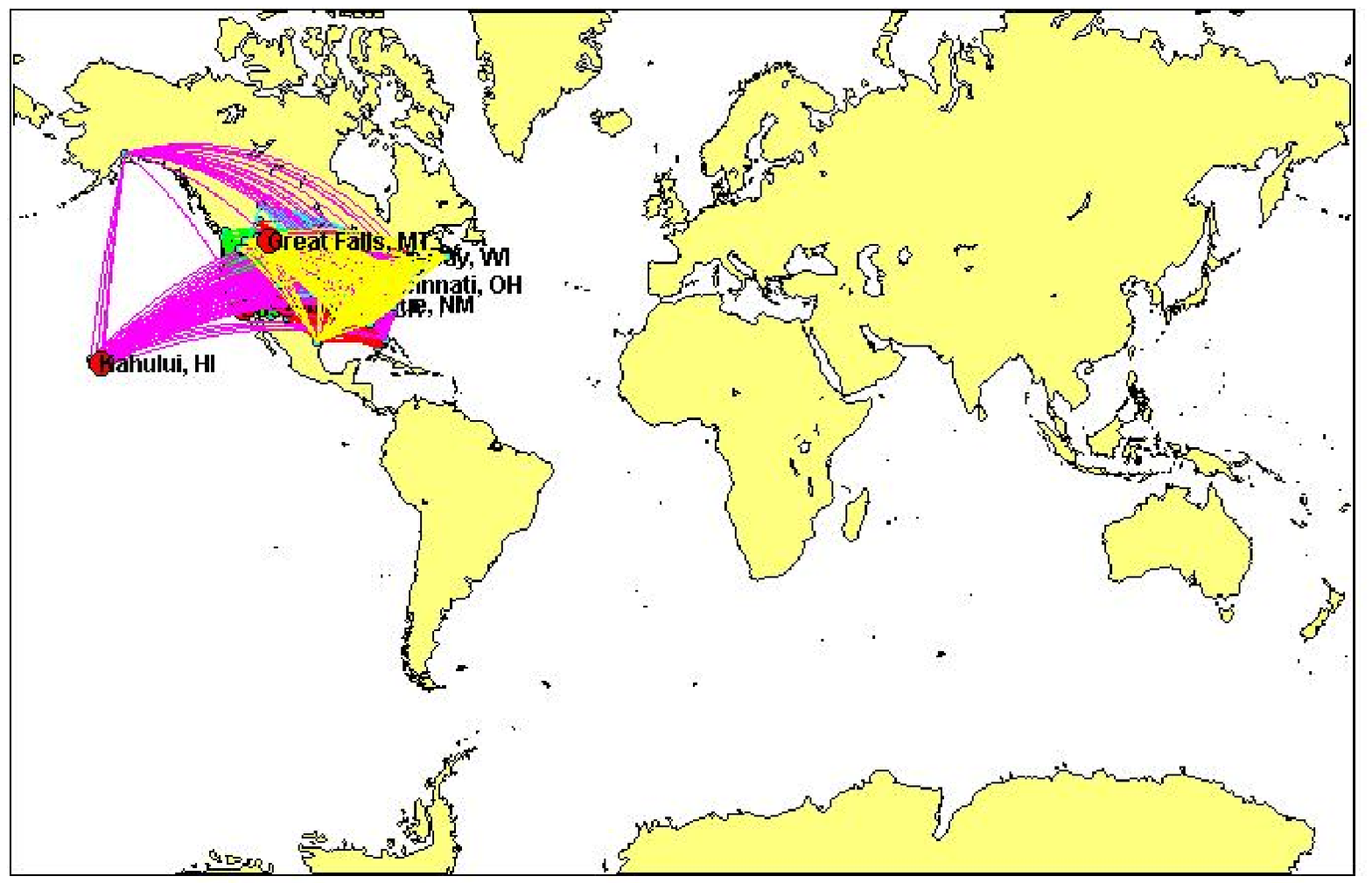}\label{fig:route_hks_30x6_feige}}
\caption{Identifying densest $k$-subgraph of air-travel routing. Row
1: Route map, and the density and CPU time evolving curves. Row 2:
The densest $30$-subgraph discovered by the three algorithms. Row 3:
Sequentially discovered six densest $30$-subgraph by the three
algorithms.}\label{fig:key_point_matches}
\end{figure}

\section{Conclusion and Future Work}
\label{sect:conclusion}

The sparse eigenvalue problem has been widely studied in machine
learning with applications such as sparse PCA. TPower is a truncated
power iteration method that approximately solves the nonconvex
sparse eigenvalue problem. Our analysis shows that when the
underlying matrix has sparse eigenvectors, under proper conditions
TPower can approximately recover the true sparse solution. The
theoretical benefit of this method is that with appropriate
initialization, the reconstruction quality depends on the restricted
matrix perturbation error at size $s$ that is comparable to the
sparsity $\bar{k}$, instead of the full matrix dimension $p$. This
explains why this method has good empirical performance. To our
knowledge, this is the first theoretical result of this kind,
although our empirical study suggests that it might be possible to
prove related sparse recovery results for some other algorithms we
have tested.

We have applied TPower to
two concrete applications: sparse PCA and the densest $k$-subgraph finding
problem. Extensive experimental results on synthetic and real-world
datasets validate the effectiveness and efficiency of the TPower algorithm.



\bibliographystyle{icml2010}
\bibliography{mybib/journals,mybib/books,mybib/proceedings,mybib/technical_reports,mybib/SparseEig}

\renewcommand{\appendixpagename}{Appendix}
\appendices
\appendixpage

\makeatletter
\renewcommand\theequation{A.\@arabic\c@equation }
\makeatother \setcounter{equation}{0}

\section{Proof of Theorem~\ref{thm:recovery}}
\label{append:proof}

We state the following standard result from the perturbation theory
of symmetric eigenvalue problem. It can be found for example in
\citep{GolubVan96}.
\begin{lemma}
  If $B$ and $B+U$ are $p \times p$ symmetric matrices, then $\forall 1 \le k \le
  p$,
  \[
  \lambda_k(B) + \lambda_p(U) \leq \lambda_k(B+U) \leq \lambda_k(B) + \lambda_1(U) ,
  \]
  where $\lambda_k(B)$ denotes the $k$-th largest eigenvalue of matrix $B$.
  \label{lem:eigenvalue}
\end{lemma}

\begin{lemma}\label{lem:perturb}
Consider set $F$ such that $\supp(\bar{x}) \subseteq F$ with
$|F|=s$. If $\spectral{E,s}\le \Delta \lambda /2$, then the ratio of
the second largest (in absolute value) to the largest eigenvalue of
sub matrix $A_F$ is no more than $\gamma(s)$. Moreover,
\[
\|\bar{x}^\top -x(F)\| \le \delta(s):=
\frac{\sqrt{2}\spectral{E,s}}{\sqrt{\spectral{E,s}^2 +
(\Delta\lambda - 2\spectral{E,s})^2}}.
\]
\end{lemma}
\begin{proof}
We may use Lemma~\ref{lem:eigenvalue} with $B=\bar{A}_F$ and $U=E_F$ to obtain
\[
\lambda_1(A_F) \geq \lambda_1(\bar{A}_F) + \lambda_p(E_F) \geq
\lambda_1(\bar{A}_F) - \spectral{E_F} \geq \lambda -\spectral{E,s}
\]
and $\forall j \ge 2$,
\[
|\lambda_j(A_F)| \leq |\lambda_j(\bar{A}_F)| + \spectral{E_F}
\leq \lambda - \Delta \lambda  +\spectral{E,s} .
\]
This implies the first statement of the lemma.

Now let $x(F)$, the largest eigenvector of $A_F$, be $\alpha \bar{x}
+ \beta x'$, where $\|\bar{x}\|_2=\|x'\|_2=1$, $\bar{x}^\top x' =0$
and $\alpha^2+\beta^2=1$, with eigenvalue $\lambda' \geq \lambda -
\spectral{E,s}$. This implies that
\[
\alpha A_F \bar{x} + \beta A_F x' = \lambda' (\alpha \bar{x} + \beta x') ,
\]
implying
\[
\alpha x'^\top A_F \bar{x} + \beta x'^\top A_F x' = \lambda' \beta  .
\]
That is,
\[
|\beta| = |\alpha| \frac{x'^\top A_F \bar{x}}{\lambda'- x'^\top A_F x'}
\leq |\alpha| \frac{|x'^\top A_F \bar{x}|}{\lambda'- x'^\top A_F x'}
= |\alpha| \frac{|x'^\top E_F \bar{x}|}{\lambda'- x'^\top A_F x'}
\leq  t |\alpha| ,
\]
where $t=\spectral{E,s}/(\Delta \lambda - 2 \spectral{E,s})$. This
implies that $\alpha^2 (1 + t^2) \geq \alpha^2 + \beta^2=1$, and
thus $\alpha^2 \geq 1/(1+t^2)$. Without loss of generality, we may
assume that $\alpha>0$, because otherwise we can replace $\bar{x}$
with $-\bar{x}$. It follows that
\[
\|x(F)-\bar{x}\|^2 = 2 - 2 x(F)^\top \bar{x} = 2 - 2 \alpha \leq 2
\frac{\sqrt{1+t^2}-1}{\sqrt{1+t^2}} \leq 2 \frac{t^2}{1+t^2} .
\]
This implies the desired bound.
\end{proof}

The following result measures the progress of untruncated power method.
\begin{lemma}\label{lemma:untruncated_power_bound}
Let $y$ be the eigenvector with the largest (in absolute value) eigenvalue of a symmetric matrix
$A$, and let $\gamma<1$ be the ratio of the second largest to
largest eigenvalue in absolute values. Given any $x$ such that $\|x\|=1$ and $y^\top x >0$; let $x' = Ax/\|Ax\|$, then
\[
|y^\top x'| \ge |y^\top x| [1+(1 - \gamma^2)(1 - (y^\top x)^2) / 2].
\]
\end{lemma}
\begin{proof}
Without loss of generality, we may assume that $\lambda_1(A)=1$ is the largest eigenvalue in absolute value,
and $|\lambda_j(A)| \leq \gamma$ when $j>1$.
We can decompose $x$ as $x= \alpha y + \beta y'$, where $y^\top y'=0$, $\|y\|=\|y'\|=1$, and $\alpha^2+\beta^2=1$.
Then $|\alpha|=|x^\top y|$. Let $z'=Ay'$, then $\|z'\| \leq \gamma $ and $y^\top z'=0$.
This means $Ax=\alpha y + \beta z'$, and
\begin{align*}
|y^\top x'| =& \frac{|y^\top Ax|}{\|Ax\|}
= \frac{|\alpha|}{\sqrt{\alpha^2 + \beta^2 \|z'\|^2}}
\geq \frac{|\alpha|}{\sqrt{\alpha^2 + \beta^2 \gamma^2}} \\
= &
\frac{|y^\top x|}{\sqrt{1 -
(1-\gamma^2)(1-(y^\top x)^2)}} \\
\ge & |y^\top x| \; [1 + (1 - \gamma^2)(1 - (y^\top x)^2) / 2] .
\end{align*}
The last inequality is due to $1/\sqrt{1-z} \geq 1 + z/2$ for $z \in [0,1)$.
This proves the desired bound.
\end{proof}

\begin{lemma}
  Consider $\bar{x}$ with $\supp(\bar{x})=\bar{F}$ and $\bar{k}=|\bar{F}|$.
  Consider $y$ and let $F= \supp(y,k)$ be the indices of $y$ with the largest $k$ absolute values.
  If $\|\bar{x}\|=\|y\|=1$, then
  \[
  |\Trunc(y,F)^\top \bar{x}| \geq |y^\top \bar{x}| - (\bar{k}/k)^{1/2}
  \min \left[1, (1+(\bar{k}/k)^{1/2}) \; (1-(y^\top\bar{x})^2) \right] .
  \]
\label{lem:trunc}
\end{lemma}
\begin{proof}
  Without loss of generality, we assume that $y^\top \bar{x}= \Delta >0$. We can also assume that
  $\Delta \geq \sqrt{\bar{k}/(\bar{k}+k)}$ because otherwise the right hand side is smaller than
  zero, and thus the result holds trivially.

  Let $F_1=\bar{F}\setminus F$, and $F_2=\bar{F}\cap F$, and $F_3=F\setminus\bar{F}$.
  Now, let $\bar{\alpha}=\|\bar{x}_{F_1}\|$, $\bar{\beta}=\|\bar{x}_{F_2}\|$,
  $\alpha= \|y_{F_1}\|$, $\beta=\|y_{F_2}\|$, and $\gamma=\|y_{F_3}\|$.
  let $k_1=|F_1|$, $k_2=|F_2|$, and $k_3=|F_3|$. It follows that
  $\alpha^2 /k_1 \leq \gamma^2/ k_3$. Therefore
  \[
  \Delta^2 \leq [\bar{\alpha} \alpha + \bar{\beta} \beta]^2 \leq \alpha^2 + \beta^2 \le 1- \gamma^2
  \leq 1- (k_3/k_1) \alpha^2 .
  \]
  This implies that
  \begin{equation}\label{inequat:alpha_delta}
  \alpha^2 \leq (k_1/k_3) (1- \Delta^2) \leq (\bar{k}/k) (1-\Delta^2) \le
  \Delta^2,
  \end{equation}
  where the second inequality follows from $\bar{k} < k$ and the last
  inequality follows from the assumption $\Delta \geq
  \sqrt{\bar{k}/(\bar{k}+k)}$.
  Now by solving the following inequality for $\bar{\alpha}$:
  \[
  \alpha \bar{\alpha} + \sqrt{1-\alpha^2} \sqrt{1-\bar{\alpha}^2} \geq \alpha \bar{\alpha} + \beta \bar{\beta} \geq
  \Delta \geq \alpha \geq \alpha \bar{\alpha} ,
  \]
  we obtain that
  \begin{equation}\label{inequat:baralpha}
  \bar{\alpha} \leq \alpha \Delta + \sqrt{1-\alpha^2} \sqrt{1-\Delta^2}
  \leq \min\left[1 , \alpha + \sqrt{1-\Delta^2} \right]
  \leq \min\left[1 , (1+(\bar{k}/k)^{1/2})\sqrt{1-\Delta^2} \right]  ,
  \end{equation}
where the second inequality follows from the Cauchy-Schwartz inequality
and $\Delta \le 1$, $\sqrt{1-\alpha^2} \le 1$, while the last
inequality follows from \eqref{inequat:alpha_delta}. Finally,
\begin{eqnarray}
|y^\top \bar{x}| - |\Trunc(y,F)^\top \bar{x}| &\leq&
|(y-\Trunc(y,F))^\top \bar{x}|  \nonumber \\
&\le& \alpha \bar\alpha \leq (\bar{k}/k)^{1/2}
  \min \left[1, (1+(\bar{k}/k)^{1/2}) \; (1-(y^\top\bar{x})^2)
  \right], \nonumber
\end{eqnarray}
where the last inequality follows from \eqref{inequat:alpha_delta} and
\eqref{inequat:baralpha}. This leads to the desired bound.
\end{proof}

Next is our main lemma, which says each step of sparse power method
improves eigenvector estimation.
\begin{lemma}
Let $s=2k+\bar{k}$. We have
\[
|\hat{x}_t^\top \bar{x}| \ge
(|x^\top_{t-1} \bar{x}| -\delta(s)) [1+(1 - \gamma(s)^2)(1 - (|x^\top_{t-1} \bar{x}|-\delta(s))^2) / 2]
- \delta(s) - (\bar{k}/k)^{1/2} .
\]
If $|x_{t-1}^\top \bar{x}| > 1/\sqrt{3} + \delta(s)$, then
\[
\sqrt{1-|\hat{x}_t^\top \bar{x}|} \leq \mu_2 \sqrt{1-|x_{t-1}^\top
\bar{x}|} + \sqrt{5} \delta(s) .
\]
\label{lem:recursion}
\end{lemma}
\begin{proof}
Let $F=F_{t-1} \cup F_t \cup \supp(\bar{x})$. Consider the following
vector
\begin{equation}\label{equat:tile_x_t}
\tilde{x}'_t = A_F x_{t-1} / \|A_F x_{t-1}\|,
\end{equation}
where $A_F$  denotes the restriction of $A$ on the rows and
columns indexed by $F$. We note that replacing $x'_t$
with $\tilde{x}'_t$ in Algorithm~\ref{alg:sparse_power_eig} does not affect the output
iteration sequence $\{x_t\}$ because of the sparsity of $x_{t-1}$
and the fact that the truncation operation is invariant to
scaling. Therefore for notation simplicity, in the following proof we will simply assume that
$x'_t$ is redefined as $x'_t=\tilde{x}'_t$ according to \eqref{equat:tile_x_t}.

Based on the above notation, we have
\begin{eqnarray*}
|x'^\top_t x(F)| &\ge& |x_{t-1}^\top x(F)|[1+(1 -
\gamma(s)^2)(1 - (x_{t-1}^\top x(F))^2) / 2] \\
&\ge& (|x^\top_{t-1} \bar{x}| - \delta(s))[1+(1 - \gamma(s)^2)(1 - (|x^\top_{t-1} \bar{x}|-\delta(s))^2) / 2] ,
\end{eqnarray*}
where the first inequality follows from
Lemma~\ref{lemma:untruncated_power_bound}, and the second is from
Lemma~\ref{lem:perturb} and $|x_{t-1}^\top \bar{x}| \geq |x^\top_{t-1} x(F)| - \delta(s)$,
and the fact that $x(1+(1-\gamma^2)(1-x^2)/2)$ is increasing when $x \in [0,1]$.
We can now use Lemma~\ref{lem:perturb} again, and the preceding inequality implies
that
\[
|x'^\top_t \bar{x}| \ge (|x^\top_{t-1} \bar{x}| -\delta(s)) [1+(1 -
\gamma(s)^2)(1 - (|x^\top_{t-1} \bar{x}|-\delta(s))^2) / 2] -
\delta(s) .
\]
Next we can apply Lemma~\ref{lem:trunc} to obtain
\begin{eqnarray*}
|\hat{x}_t^\top \bar{x}| &\ge&|{x'_t}^\top \bar{x}| - (\bar{k}/k)^{1/2}\\
&\ge&
(|x^\top_{t-1} \bar{x}| -\delta(s)) [1+(1 - \gamma(s)^2)(1 - (|x^\top_{t-1} \bar{x}|-\delta(s))^2) / 2]
- \delta(s) - (\bar{k}/k)^{1/2} .
\end{eqnarray*}
This leads to the first desired inequality.

Next we will prove the second inequality. Without loss of generality and for simplicity, we may assume that
$x'^\top_t x(F) \geq 0$ and $x_{t-1}^\top \bar{x} \geq 0$, because otherwise we can simply do appropriate sign changes
in the proof.
We obtain from Lemma~\ref{lemma:untruncated_power_bound} that
\[
x'^\top_t x(F)\ge x_{t-1}^\top x(F) \; [1+(1 -\gamma(s)^2)(1 - (x_{t-1}^\top x(F))^2) / 2] .
\]
This implies that
\begin{align*}
[1- x'^\top_t x(F)] \le& [1 - x_{t-1}^\top x(F)] \; [1-(1 -\gamma(s)^2)(1 +x_{t-1}^\top x(F)) (x_{t-1}^\top x(F)) / 2] \\
\le& [1 - x_{t-1}^\top x(F)] \; [1-0.45 (1 -\gamma(s)^2) ] ,
\end{align*}
where in the derivation of the second inequality, we have used Lemma~\ref{lem:perturb} and the assumption of the lemma that implies $x_{t-1}^\top x(F) \geq x_{t-1}^\top \bar{x} - \delta(s) \geq 1/\sqrt{3}$.
We thus have
\[
\|x'_t - x(F)\| \le \|x_{t-1}- x(F)\| \; \sqrt{1-0.45 (1 -\gamma(s)^2)} .
\]
Therefore using Lemma~\ref{lem:perturb}, we have
\[
\|x'_t - \bar{x} \| \le \|x_{t-1}- \bar{x}\| \; \sqrt{1-0.45 (1 -\gamma(s)^2)}
+ 2 \delta(s) .
\]
This is equivalent to
\[
\sqrt{1-|{x'_t}^\top \bar{x}|} \le \sqrt{1-|x_{t-1}^\top \bar{x}|} \sqrt{1-0.45 (1 -\gamma(s)^2)} + \sqrt{2} \delta(s) .
\]
Next we can apply Lemma~\ref{lem:trunc} and use $\bar{k}/k \leq 0.25$ to obtain
\begin{eqnarray*}
\sqrt{1-|\hat{x}_t^\top \bar{x}|} &\le&
\sqrt{1 - |{x'_t}^\top \bar{x}| +  ((\bar{k}/k)^{1/2}+\bar{k}/k) (1- |{x'_t}^\top \bar{x}|^2)} \\
&\le& \sqrt{1 - |{x'_t}^\top \bar{x}|} \sqrt{1 + 3 (\bar{k}/k)^{1/2}} \\
&\le& \mu_2 \sqrt{1 - |{x_{t-1}}^\top \bar{x}|} + \sqrt{5} \delta(s) .
\end{eqnarray*}
This proves the second desired inequality.
\end{proof}

\noindent {\bf Proof of Theorem~\ref{thm:recovery}}

\noindent We know if $|x_{t-1}^\top\bar{x}| \geq u+\delta(s)$, then Lemma~\ref{lem:recursion} implies:
\begin{align*}
|\hat{x}_t^\top \bar{x}| \ge &
(|x_{t-1}^\top\bar{x}|-\delta(s)) [1 + (1-\gamma(s)^2)(1-(|x_{t-1}^\top\bar{x}|-\delta(s))^2) / 2] - (\delta(s)+ (\bar{k}/k)^{1/2}) \\
\ge & u [1 + (1-\gamma(s)^2)(1-u^2) / 2] - (\delta(s)+ (\bar{k}/k)^{1/2}) \geq u +\delta(s) .
\end{align*}
The first inequality uses Lemma~\ref{lem:recursion}; the second inequality uses
$z (1+ (1-\gamma^2) (1-z^2)/2)$ is an increasing function of $z \in [0,1]$;
and the third inequality uses the assumption of $u$ in the theorem.
This implies (by an easy induction argument) that we have $|\hat{x}_t^\top \bar{x}| \geq u+\delta$ for all $t \geq 0$.

Now we can prove the theorem by induction. The bound clearly holds
at $t=0$. Assume it holds at some $t-1$. If we have
$|\hat{x}_{t-1}^\top \bar{x}| \leq 1/\sqrt{3}+ \delta(s)$, then
since $z(1-z^2)$ is increasing in $[0,1/\sqrt{3}]$, and from
Lemma~\ref{lem:recursion} we have
\begin{align*}
|\hat{x}_t^\top \bar{x}| \ge &
(|x_{t-1}^\top\bar{x}|-\delta(s))[1 + (1-\gamma(s)^2) (1-(|x_{t-1}^\top \bar{x}|-\delta(s))^2) / 2 - (\delta(s)+ (\bar{k}/k)^{1/2}) \\
\ge & |x_{t-1}^\top\bar{x}| -\delta(s) + (1-\gamma(s)^2) u (1-u^2) / 2 - (\delta(s)+ (\bar{k}/k)^{1/2}) \\
\ge & |x_{t-1}^\top\bar{x}| + \mu_1 .
\end{align*}
Combing this inequality with $|x_t^\top \bar{x}| = |\hat{x}_t^\top
\bar{x}| / \|\hat{x}_t\| \ge |\hat{x}_t^\top \bar{x}|$ we get
\[
1- |x_t^\top \bar{x}| \leq 1- |\hat{x}_t^\top \bar{x}| \le (1-\mu_1)
[1-|x_{t-1}^\top\bar{x}|] ,
\]
which implies the theorem at $t$.

If $|\hat{x}_{t-1}^\top \bar{x}| \geq 1/\sqrt{3}+\delta(s)$, then we have from Lemma~\ref{lem:recursion}
\begin{align*}
\sqrt{1- |x_t^\top \bar{x}|} \le \sqrt{1- |\hat{x}_t^\top \bar{x}|}
\le \mu_2 \sqrt{1- |x_{t-1}^\top \bar{x}|} +\sqrt{5} \delta(s) ,
\end{align*}
which implies the theorem at $t$.
This finishes induction.

\end{document}